\documentclass[11pt]{amsart}

\usepackage{microtype}
\usepackage{graphicx}
\usepackage{subfigure}
\usepackage{booktabs} 
\usepackage{bm}
\usepackage{amsthm}
\usepackage{amsmath}
\usepackage{amssymb}
\usepackage{cite}
\usepackage{color}
\usepackage{multirow}
\usepackage{scalefnt}
\usepackage{lscape}
\usepackage{url}
\usepackage{bm}
\usepackage{algorithmic}
\usepackage{algorithm}

\newtheorem{thm}{Theorem}[section]

\newtheorem{lem}{Lemma}[section]

\numberwithin{equation}{section}

\title[Minimization of SFO Complexity
 of Adaptive Methods for Nonconvex Optimization]
 {Minimization of Stochastic First-order Oracle Complexity
 of Adaptive Methods for Nonconvex Optimization}

\author[H. Iiduka]{Hideaki Iiduka}

\address[H. Iiduka]{Department of Computer Science
Meiji University
1-1-1 Higashimita, Tama-ku, Kawasaki-shi, Kanagawa 214-8571 Japan}
\email{{\tt iiduka@cs.meiji.ac.jp}}


\begin{document}

\begin{abstract}
Numerical evaluations have definitively shown that, for deep learning optimizers such as stochastic gradient descent, momentum, and adaptive methods, the number of steps needed to train a deep neural network halves for each doubling of the batch size and that there is a region of diminishing returns beyond the critical batch size. 
In this paper, we determine the actual critical batch size by using the global minimizer of the stochastic first-order oracle (SFO) complexity of the optimizer. 
To prove the existence of the actual critical batch size, we set the lower and upper bounds of the SFO complexity and prove that there exist critical batch sizes in the sense of minimizing the lower and upper bounds. 
This proof implies that, if the SFO complexity fits the lower and upper bounds, then the existence of these critical batch sizes demonstrates the existence of the actual critical batch size. 
We also discuss the conditions needed for the SFO complexity to fit the lower and upper bounds and provide numerical results that support our theoretical results.
\end{abstract}

\maketitle

\section{Introduction}
\label{sec:1}
\subsection{Background}
\label{subsec:1.1}
Adaptive methods are powerful deep learning optimizers used to find the model parameters of deep neural networks that minimize the expected risk and empirical risk loss functions \cite[Section 4]{bottou}. Specific adaptive methods are, for example, adaptive gradient (AdaGrad) \cite{adagrad}, root mean square propagation (RMSProp) \cite{rmsprop}, adaptive moment estimation (Adam) \cite{adam}, and adaptive mean square gradient (AMSGrad) \cite{reddi2018} (Table 2 in \cite{Schmidt2021} lists the main deep learning optimizers).

Deep learning optimizers have been widely studied from the viewpoints of both theory and practice. While the convergence and convergence rate of deep learning optimizers have been theoretically studied for convex optimization \cite{zin2010,adam,reddi2018,luo2019,dun2020}, theoretical investigation of deep learning optimizers for nonconvex optimization is needed so that such optimizers can be put into practical use in deep learning \cite{kxu2015,ar2017,vas2017}.

The convergence of adaptive methods has been studied for nonconvex optimization \cite{spider,chen2019,ada,iiduka2021}, and the convergence of stochastic gradient descent (SGD) methods has been studied for nonconvex optimization \cite{feh2020,chen2020,sca2020,loizou2021}. Chen et al. showed that generalized Adam, which includes the heavy-ball method, AdaGrad, RMSProp, AMSGrad, and AdaGrad with first-order momentum (AdaFom), has an $\mathcal{O}(\log K/\sqrt{K})$ convergence rate when a decaying learning rate ($\alpha_k = 1/\sqrt{k}$) is used \cite{chen2019}. AdaBelief (which adapts the step size in accordance with the belief in the observed gradients) using $\alpha_k = 1/\sqrt{k}$ has an $\mathcal{O}(\log K/\sqrt{K})$ convergence rate \cite{ada}. A method that unifies adaptive methods such as AMSGrad and AdaBelief has been shown to have a convergence rate of $\mathcal{O}(1/\sqrt{K})$ when $\alpha_k = 1/\sqrt{k}$ \cite{iiduka2021}, which improves on the results of \cite{chen2019,ada}.

\begin{figure*}[htbp]
\begin{tabular}{cc}
\begin{minipage}[t]{0.45\hsize}
\centering
\includegraphics[width=1\textwidth]{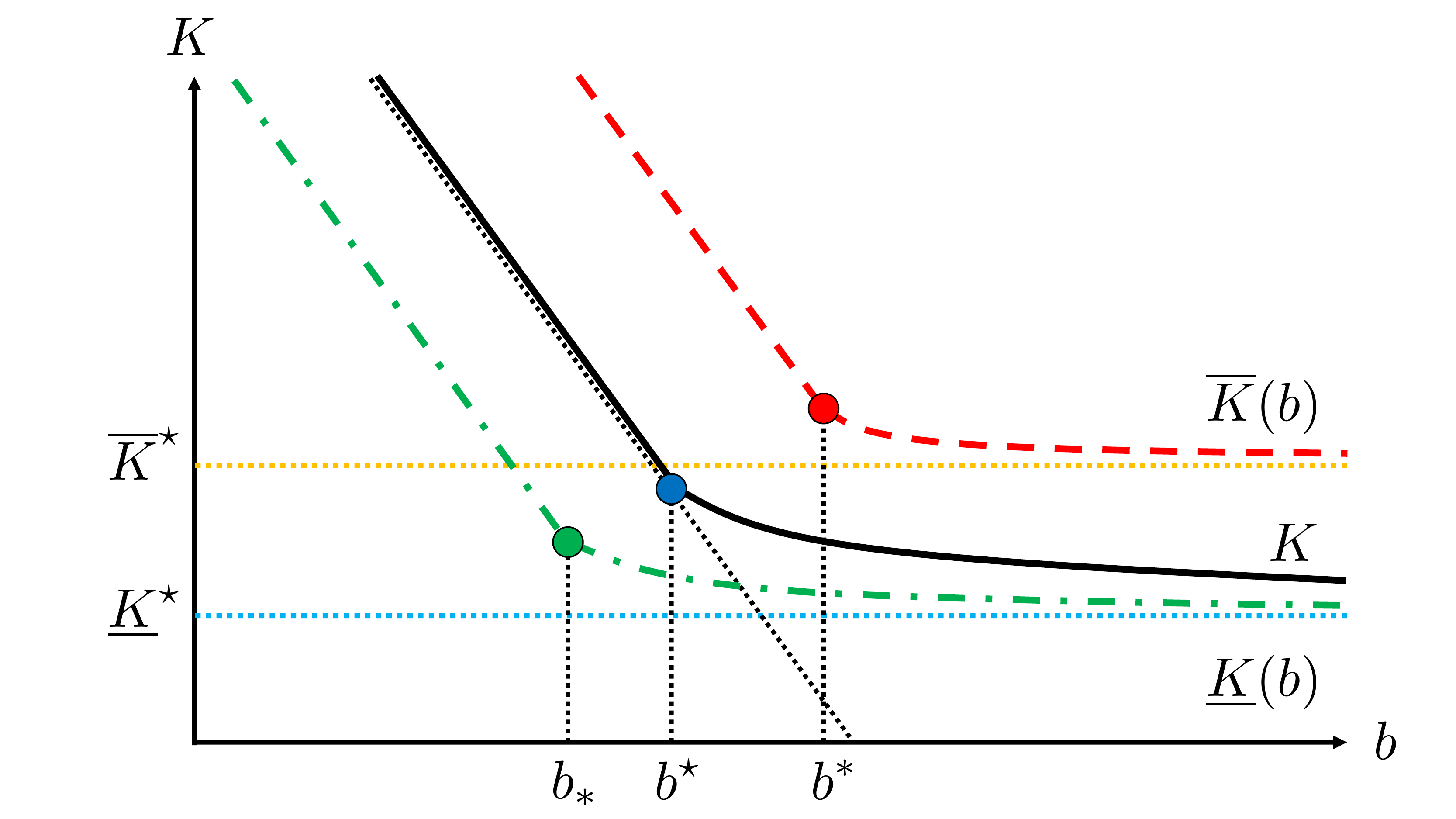}
\caption{Relationship between lower bound $\underline{K}(b)$ and upper bound $\overline{K}(b)$ ($b^\star$ denotes critical batch size).}
\label{fig:1}
\end{minipage} &
\begin{minipage}[t]{0.45\hsize}
\centering
\includegraphics[width=1\textwidth]{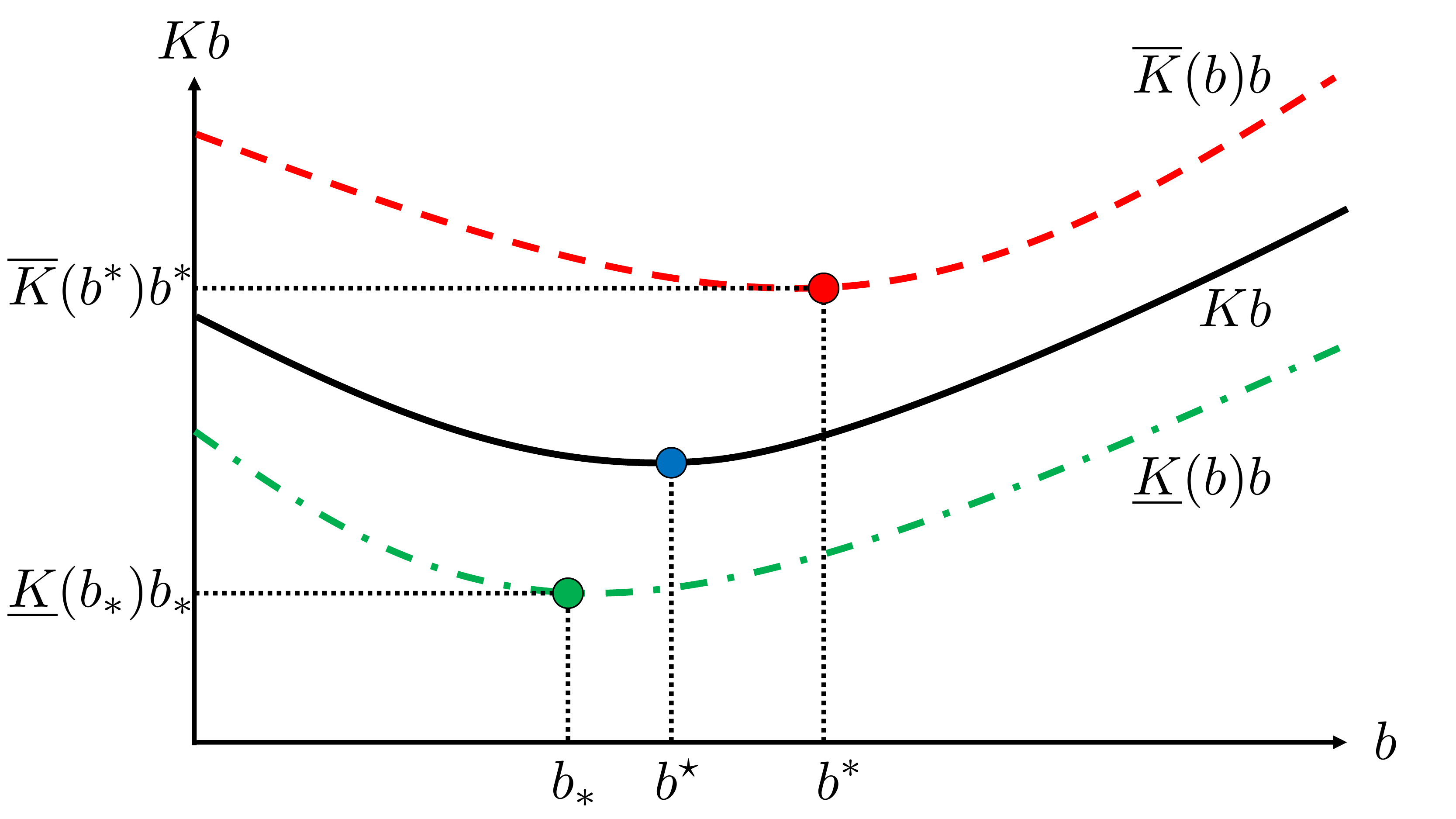}
\caption{Relationship between $\underline{K}(b) b$ and $\overline{K}(b) b$ ($b_*$ and $b^*$ are global minimizers of $\underline{K}(b) b$ and $\overline{K}(b) b$ that always exist).}
\label{fig:2}
\end{minipage}
\end{tabular}
\end{figure*}

There have been several practical studies of deep learning optimizers. Shallue et al. \cite{shallue2019} studied how increasing the batch size affects the performances of SGD, SGD with momentum \cite{polyak1964,momentum}, and Nesterov momentum \cite{nest1983,sut2013}. Zhang et al. \cite{zhang2019} studied the relationship between batch size and performance for Adam and K-FAC (Kronecker-factored approximate curvature \cite{kfac}). In both studies, it was numerically shown that increasing the batch size tends to decrease the number of steps needed for training deep neural networks, but with diminishing returns. Moreover, it was shown that SGD with momentum and Nesterov momentum can exploit larger batches than SGD \cite{shallue2019} and that K-FAC and Adam can exploit larger batches than SGD with momentum \cite{zhang2019}. Thus, it has been shown that momentum and adaptive methods can greatly reduce the number of steps needed for training deep neural networks \cite[Figure 4]{shallue2019}, \cite[Figure 5]{zhang2019}. Smith et al. \cite{l.2018dont} numerically showed that using enormous batches leads to reductions in the number of parameter updates and model training time.

\subsection{Motivation}
As described in Section \ref{subsec:1.1}, there have been theoretical studies of the relationship between the learning rate and convergence of deep learning optimizers. Furthermore, as also described in Section \ref{subsec:1.1}, the practical performance of a deep learning optimizer strongly depends on the batch size. Our goal in this paper is to clarify theoretically the relationship between batch size and the performance of deep learning optimizers. We focus on the relationship between the batch size and the number of steps needed for nonconvex optimization of several adaptive methods.

To explicitly show our contribution, we characterize the diminishing returns reported in \cite{shallue2019,zhang2019} by using the stochastic first-order oracle (SFO) complexities of deep learning optimizers. We define the SFO complexity of a deep learning optimizer as $Kb$ on the basis of the number of steps $K$ needed for training the deep neural network and on the batch size $b$ used in the optimizer. Let $b^\star$ be a critical batch size such that there are diminishing returns for all batch sizes beyond $b^\star$, as asserted in \cite{shallue2019,zhang2019}. Thanks to the numerical evaluations in \cite{shallue2019,zhang2019}, we know that there is a positive number $C$ such that 
\begin{align}\label{fact}
Kb \approx 2^C \text{ for } b \leq b^\star \text{ and }
Kb \geq 2^C \text{ for } b \geq b^\star,
\end{align}
where $K$ and $b$ are defined for $i,j\in\mathbb{N}$ by $K = 2^i$ and $b = 2^j$. For example, Figure 5(a) in \cite{shallue2019} shows $C \approx 16$ and $b^\star \approx 2^{8}$ for the momentum method used to train a convolutional neural network on the MNIST dataset. This implies that, while SFO complexity $Kb$ initially is not almost changed (i.e., $K$ halves for each doubling of $b$), $Kb$ is minimized at critical batch size $b^\star$, and there are diminishing returns once the batch size exceeds $b^\star$. Therefore, we can see that there are diminishing returns beyond the critical batch size $b^\star$ that minimizes SFO complexity. 

\subsection{Goal}
The goal of this paper is to develop a theory demonstrating the existence of critical batch sizes, which was shown numerically by Shallue et al. and Zhang et al. \cite{shallue2019,zhang2019}.

To reach this goal, we first examine the relationship between the number of steps $K$ needed for nonconvex optimization (Section \ref{sec:2}) and the batch size $b$ used for each adaptive method. We show that, under certain assumptions, the lower bound $\underline{K}$ and the upper bound $\overline{K}$ of $K$ are rational functions of batch size $b$; i.e., $\underline{K}$ and $\overline{K}$ have the following forms (see Section \ref{sec:3} for details):
\begin{align}\label{rational}
\frac{A b}{\{ \epsilon^2 - (C+D)\}b - B}
&=:
\underline{K}(b) \leq K\\ 
&\leq
\overline{K}(b) 
:= \frac{E b}{(\delta \epsilon^2 + G) b - F},\nonumber 
\end{align}
where $A$, $B$, $C$, $D$, $E$, $F$, and $G$ are positive constants depending on the parameters (e.g., constant learning rate $\alpha$ and momentum coefficient $\beta$) used in each adaptive method, and $\epsilon, \delta > 0$ are the precisions for nonconvex optimization. $\underline{K}$ and $\overline{K}$ defined by \eqref{rational} are monotone decreasing and convex for batch size $b$. Accordingly, $\underline{K}$ and $\overline{K}$ decrease with an increase in batch size $b$. For minimizing the empirical average loss function for the given number of samples $n$ (see \eqref{empirical}), $\underline{K}$ and $\overline{K}$ are minimized at $b = n$. We can see that there are asymptotic values $\underline{K}^\star = A/\{ \epsilon^2 - (C+D)\}$ and $\overline{K}^\star = E/(\delta \epsilon^2 + G)$ of $\underline{K}$ and $\overline{K}$. 

Figure \ref{fig:1} visualizes the relationship between $\underline{K}(b)$ and $\overline{K}(b)$. The number of steps $K$ needed for nonconvex optimization is inversely proportional to a batch size smaller than the critical batch size $b^\star$, and there are diminishing returns beyond the critical batch $b^\star$, as shown by the numerical results of Shallue et al. and Zhang et al.  \cite{shallue2019,zhang2019}. Moreover, the critical batch size $b^\star$ minimizes SFO complexity $Kb$, as seen in \eqref{fact}. We thus define the critical batch size by using the global minimizer of SFO complexity $K(b) b$.

From the forms of $\underline{K}(b)$ and $\overline{K}(b)$ in \eqref{rational}, we have the forms of $\underline{K}(b) b$ and $\overline{K}(b)b$:
\begin{align}\label{rational2}
\frac{A b^2}{\{ \epsilon^2 - (C+D)\}b - B}
&=
\underline{K}(b) b
\leq K b\\ 
&\leq
\overline{K}(b) b = \frac{E b^2}{(\delta \epsilon^2 + G) b - F}.\nonumber
\end{align}
We can show that $\underline{K}(b) b$ and $\overline{K}(b) b$ are convex; i.e., there are global minimizers, denoted by $b_*$ and $b^*$, of $\underline{K}(b) b$ and $\overline{K}(b) b$, as seen in Figure \ref{fig:2} (see Section \ref{sec:3} for details). Therefore, we can regard batch sizes $b_*$ and $b^*$ as the critical batch sizes in the sense of minimizing $\underline{K}(b) b$ and $\overline{K}(b) b$ defined by \eqref{rational2}.
Suppose that $\underline{K}(b) b$ and $\overline{K}(b) b$ defined by \eqref{rational2} approximate the actual SFO complexity $Kb$; i.e., 
\begin{align*}
&\frac{2B}{\epsilon^2 - (C+D)} = b_* \approx b^* 
= \frac{2F}{\delta \epsilon^2 + G}
\text{ and}\\ 
&\frac{4AB}{\{ \epsilon^2 - (C+D) \}^2}
=
\underline{K}(b_*) b_* \approx 
\overline{K}(b^*) b^*
= 
\frac{4EF}{(\delta \epsilon^2 + G)^2},
\end{align*}
which implies that 
\begin{align}\label{condition}
\begin{split}
&\underbrace{(C+D) F}_{\approx (M_1 \alpha + M_2)(M_3 \alpha - M_4)} + 
\underbrace{BG}_{\approx M_5 \alpha} \approx (F - \delta B) \epsilon^2 
\text{ and}\\
&\underbrace{(C+D)}_{\approx M_1 \alpha + M_2} \sqrt{EF} 
+ G \sqrt{AB} \approx (\sqrt{EF}- \delta \sqrt{AB})\epsilon^2,
\end{split}
\end{align}
where $M_i$ ($i=1,2,3,4,5$) are positive constants that do not depend on a constant learning rate $\alpha$ (see Section \ref{sec:3} for details). Hence, if precisions $\epsilon$ and $\delta$ are small, using a sufficiently small learning rate $\alpha$ could be used to approximate $K(b)b$ (Indeed, small learning rates were practically used by Shallue et al. \cite{shallue2019}). Then, the demonstration of the existence of critical batch sizes $b_*$ and $b^*$ leads to the existence of the actual critical batch $b^\star \approx b_* \approx b^*$.
 
\section{Nonconvex Optimization and Deep Learning Optimizers}
\label{sec:2}
\subsection{Nonconvex optimization}
Let $\mathbb{R}^d$ be a $d$-dimensional Euclidean space with inner product $\langle \bm{x},\bm{y} \rangle := \bm{x}^\top \bm{y}$ inducing norm $\| \bm{x}\|$, and let $\mathbb{N}$ be the set of nonnegative integers. Let $\mathbb{S}_{++}^d$ be the set of $d \times d$ symmetric positive-definite matrices, and let $\mathbb{D}^d$ be the set of $d \times d$ diagonal matrices: $\mathbb{D}^d = \{ M \in \mathbb{R}^{d \times d} \colon M = \mathsf{diag}(x_i), \text{ } x_i \in \mathbb{R} \text{ } (i\in [d] := \{1,2,\ldots,d\}) \}$.

The mathematical model used here is based on that of Shallue et al. \cite{shallue2019}. Given a parameter $\bm{\theta}$ in a set $\mathcal{S} \subset \mathbb{R}^d$ and given a data point $z$ in a data domain $Z$, a machine learning model provides a prediction for which the quality is estimated using a differentiable nonconvex loss function $\ell(\bm{\theta};z)$. We minimize the expected loss defined for all $\bm{\theta} \in \mathcal{S}$ by using
\begin{align}\label{expected}
L(\bm{\theta}) = \mathbb{E}_{z \sim \mathcal{D}} 
[\ell(\bm{\theta};z) ]
= \mathbb{E}[ \ell_{\xi} (\bm{\theta}) ],
\end{align}
where $\mathcal{D}$ is the probability distribution over $Z$, $\xi$ denotes a random variable with distribution function $\mathrm{P}$, and $\mathbb{E}[\cdot]$ denotes the expectation taken with respect to $\xi$. A particularly interesting example of \eqref{expected} is the empirical average loss defined for all $\bm{\theta} \in \mathcal{S}$:
\begin{align}\label{empirical}
L(\bm{\theta}; S) = \frac{1}{n} \sum_{i\in [n]} \ell(\bm{\theta};z_i)
= \frac{1}{n} \sum_{i\in [n]} \ell_i(\bm{\theta}),
\end{align}
where $S = (z_1, z_2, \ldots, z_n)$ denotes the training set, $\ell_i (\cdot) := \ell(\cdot;z_i)$ denotes the loss function corresponding to the $i$-th training data instance $z_i$, and $[n] := \{1,2,\ldots,n\}$. Our main objective is to find a {\em local minimizer} of $L$ over $\mathcal{S}$, i.e., a point $\bm{\theta}^\star \in \mathcal{S}$ that belongs to the set
\begin{align}\label{VI}
\begin{split}
&\mathrm{VI}\left(\mathcal{S},\nabla L \right)\\
&\quad := 
\left\{ 
\bm{\theta}^\star \in \mathcal{S} \colon
(\bm{\theta}^\star - \bm{\theta})^\top \nabla L(\bm{\theta}^\star) \leq 0
\text{ } \left(\bm{\theta} \in \mathcal{S} \right)
\right\},
\end{split}
\end{align}
where $\nabla L \colon \mathbb{R}^d \to \mathbb{R}^d$ denotes the gradient of $L$. The inequality $(\bm{\theta}^\star - \bm{\theta})^\top \nabla L(\bm{\theta}^\star) \leq 0$ is called the {\em variational inequality} of $\nabla L$ over $\mathcal{S}$ \cite{facc1}.

\subsection{Deep learning optimizers}
We assume that there exists SFO such that, for a given $\bm{\theta} \in \mathcal{S}$, it returns a stochastic gradient $\mathsf{G}_{\xi}(\bm{\theta})$ of the function $L$ defined by \eqref{expected}, where a random variable $\xi$ is supported on $\Xi$ and does not depend on $\bm{\theta}$. Throughout this paper, we assume three standard conditions:
\begin{enumerate}
\item[(S1)] $L \colon \mathbb{R}^d \to \mathbb{R}$ defined by \eqref{expected} is continuously differentiable.
\item[(S2)] Let $(\bm{\theta}_k)_{k\in \mathbb{N}} \subset \mathcal{S}$ be the sequence generated by a deep learning optimizer. For each iteration $k$, 
\begin{align}\label{gradient}
\mathbb{E}_{\xi_k}[ \mathsf{G}_{\xi_k}(\bm{\theta}_k)] = \nabla L(\bm{\theta}_k),
\end{align}
where $\xi_0, \xi_1, \ldots$ are independent samples, and the random variable $\xi_k$ is independent of $(\bm{\theta}_l)_{l=0}^k$. There exists a nonnegative constant $\sigma^2$ such that 
\begin{align}\label{sigma}
\mathbb{E}_{\xi_k}\left[ \left\|\mathsf{G}_{\xi_k}(\bm{\theta}_k) - 
\nabla L(\bm{\theta}_k) \right\|^2 \right] \leq \sigma^2.
\end{align}
\item[(S3)] For each iteration $k$, the deep learning optimizer samples a batch $B_{k}$ of size $b$ and estimates the full gradient $\nabla L$ as 
\begin{align*}
\nabla L_{B_k} (\bm{\theta}_k)
:= \frac{1}{b} \sum_{i\in [b]} \mathsf{G}_{\xi_{k,i}}(\bm{\theta}_k),
\end{align*}
where $\xi_{k,i}$ is the random variable generated by the $i$-th sampling in the $k$-th iteration. 
\end{enumerate}
In the case where $L$ is defined by \eqref{empirical}, we have that, for each $k$, $B_k \subset [n]$ and 
\begin{align*}
\nabla L_{B_k} (\bm{\theta}_k)
= \frac{1}{b} \sum_{i \in B_k} \nabla \ell_{i} (\bm{\theta}_k).
\end{align*}

We consider the following algorithm (Algorithm \ref{algo:1}), which is a unified algorithm for most deep learning optimizers,\footnote{We define $\bm{x} \odot \bm{x}$ for $\bm{x} := (x_i)_{i=1}^d \in \mathbb{R}^d$ by $\bm{x} \odot \bm{x} := (x_i^2)_{i=1}^d \in \mathbb{R}^d$. $\mathrm{C}(\cdot, l,u) \colon \mathbb{R} \to \mathbb{R}$ in AMSBound ($l,u \in \mathbb{R}$ with $l \leq u$ are given) is defined for all $x \in \mathbb{R}$ by
\begin{align*}
\mathrm{C}(x,l,u) := 
\begin{cases}
l &\text{ if } x < l,\\
x &\text{ if } l \leq x \leq u,\\
u &\text{ if } x > u.
\end{cases}
\end{align*}}
including Momentum \cite{polyak1964}, AMSGrad \cite{reddi2018}, AMSBound \cite{luo2019}, Adam \cite{adam}, and AdaBelief \cite{ada}, which are listed in Table \ref{table:ex}.

\begin{algorithm} 
\caption{Deep learning optimizer} 
\label{algo:1} 
\begin{algorithmic}[1] 
\REQUIRE
$\alpha \in (0,+\infty)$, 
$\beta \in [0,1)$, 
$\gamma \in [0,1)$
\STATE
$k \gets 0$, $\bm{\theta}_0 \in\mathcal{S}$, $\bm{m}_{-1} := \bm{0}$, 
$\mathsf{H}_0 \in \mathbb{S}_{++}^d \cap \mathbb{D}^d$
\LOOP 
\STATE 
$\bm{m}_k := \beta \bm{m}_{k-1} + (1-\beta) \nabla L_{B_k}(\bm{\theta}_k)$
\STATE
$\displaystyle{\hat{\bm{m}}_k := (1-{\gamma}^{k+1})^{-1}\bm{m}_k}$
\STATE
$\mathsf{H}_k \in \mathbb{S}_{++}^d \cap \mathbb{D}^d$ \text{ } (see Table \ref{table:ex} for examples of $\mathsf{H}_k$)
\STATE 
Find $\bm{\mathsf{d}}_k \in \mathbb{R}^d$ that solves $\mathsf{H}_k \bm{\mathsf{d}} = - \hat{\bm{m}}_k$
\STATE 
$\bm{\theta}_{k+1} := \bm{\theta}_k + \alpha \bm{\mathsf{d}}_k$
\STATE $k \gets k+1$
\ENDLOOP 
\end{algorithmic}
\end{algorithm}

Let $\xi_k = (\xi_{k,1}, \xi_{k,2},\ldots, \xi_{k,b})$ be the random samplings in the $k$-th iteration, and let $\xi_{[k]} = (\xi_0, \xi_1, \ldots, \xi_k)$ be the history of process $\xi_0, \xi_1, \ldots$ to time step $k$. The adaptive methods listed in Table \ref{table:ex} all satisfy the following conditions:
\begin{enumerate}
\item[(A1)] $\mathsf{H}_k := \mathsf{diag}(h_{k,i})$ depends on $\xi_{[k]}$, and $h_{k+1,i} \geq h_{k,i}$ for all $k\in\mathbb{N}$ and all $i\in [d]$;
\item[(A2)] For all $i\in [d]$, a positive number $H_i$ exists such that $\sup_{k \in \mathbb{N}} \mathbb{E}[h_{k,i}] \leq H_i$.
\end{enumerate}
We define $H := \max_{i\in [d]} H_i$. The optimizers in Table \ref{table:ex} obviously satisfy (A1). Previously reported results \cite[p.29]{chen2019}, \cite[p.18]{ada}, and \cite{iiduka2021} show that $(\mathsf{H}_k)_{k\in\mathbb{N}}$ in Table \ref{table:ex} satisfies (A2). For example, AMSGrad (resp. Adam) satisfies (A2) with $H = \sqrt{M}$ (resp. $H = \sqrt{M/(1-\zeta)}$), where $M := \sup_{k\in \mathbb{N}} \|\nabla L_{B_k}(\bm{\theta}_k) \odot \nabla L_{B_k}(\bm{\theta}_k)\| < + \infty$.

\begin{table}[htbp]
\caption{Examples of $\mathsf{H}_k \in \mathbb{S}_{++}^d \cap \mathbb{D}^d$ (step 5) in Algorithm \ref{algo:1} ($\eta,\zeta \in [0,1)$)}
\label{table:ex}
\begin{tabular}{l|l}
\hline
& $\mathsf{H}_k$ \\ \hline \hline
SGD 
& $\mathsf{H}_k$ is the identity matrix. \\
($\beta = \gamma = 0$)
&
\\ \hline
Momentum 
&
$\mathsf{H}_k$ is the identity matrix. \\
\cite{polyak1964}
& \\ 
($\gamma = 0$)
&
\\
\hline
AMSGrad 
&
$\bm{p}_k = \nabla L_{B_k}(\bm{\theta}_k) \odot \nabla L_{B_k}(\bm{\theta}_k)$ \\
\cite{reddi2018}
&
$\bm{v}_k = \eta \bm{v}_{k-1} + (1-\eta) \bm{p}_k$ \\ 
($\gamma = 0$)
&
$\hat{\bm{v}}_k = (\max \{ \hat{v}_{k-1,i}, v_{k,i} \})_{i=1}^d$ \\
&
$\mathsf{H}_k = \mathsf{diag} (\sqrt{\hat{v}_{k,i}})$ \\ \hline
AMSBound &
$\bm{p}_k = \nabla L_{B_k}(\bm{\theta}_k) \odot \nabla L_{B_k}(\bm{\theta}_k)$ \\
\cite{luo2019}
& 
$\bm{v}_k = \eta \bm{v}_{k-1} + (1- \eta) \bm{p}_k$ \\
($\gamma = 0$) &
$\hat{\bm{v}}_k = (\max \{ \hat{v}_{k-1,i}, v_{k,i} \})_{i=1}^d$ \\
&
$\tilde{{v}}_{k,i} 
= \mathrm{C}\left( \frac{1}{\sqrt{\hat{v}_{k,i}}}, l_k, u_k \right)^{-1}$ \\
&
$\mathsf{H}_k = \mathsf{diag} (\tilde{v}_{k,i})$ \\ \hline
Adam 
&
$\bm{p}_k = \nabla L_{B_k}(\bm{\theta}_k) \odot \nabla L_{B_k}(\bm{\theta}_k)$\\
\cite{adam}
&
$\bm{v}_k = \eta \bm{v}_{k-1} + (1- \eta) \bm{p}_k$\\
($v_{k,i} \leq v_{k+1,i}$)
&
$\bar{\bm{v}}_k = \frac{\bm{v}_k}{1- \zeta^k}$ \\
&
$\mathsf{H}_k = \mathsf{diag} (\sqrt{\bar{v}_{k,i}})$ \\ \hline
AdaBelief 
& 
$\tilde{\bm{p}}_k = \nabla L_{B_k}(\bm{\theta}_k) - \bm{m}_k$ \\
\cite{ada}
&
$\tilde{\bm{s}}_k = \tilde{\bm{p}}_k \odot \tilde{\bm{p}}_k$ \\
($s_{k,i} \leq s_{k+1,i}$) &
$\bm{s}_k = \eta \bm{v}_{k-1} + (1-\eta) \tilde{\bm{s}}_k$ \\
&
$\hat{\bm{s}}_k = \frac{\bm{s}_k}{1- \zeta^k}$ \\
&
$\mathsf{H}_k = \mathsf{diag} (\sqrt{\hat{s}_{k,i}})$ \\ \hline
\end{tabular}
\end{table}

The theoretical performance metric used to approximate a local minimizer in $\mathrm{VI}(\mathcal{S}, \nabla L)$ (see \eqref{VI}) is an $\epsilon$-approximation of the sequence $(\bm{\theta}_k)_{k\in\mathbb{N}}$ generated by Algorithm \ref{algo:1} in the sense of the mean value of $(\bm{\theta}_k - \bm{\theta})^\top \nabla L(\bm{\theta}_k)$ ($k\in [K]$); that is, 
\begin{align}\label{evaluation}
\begin{split}
\delta \epsilon^2
&\leq
\mathrm{V}(K, \bm{\theta}) := \frac{1}{K} \sum_{k\in [K]} \mathbb{E}\left[(\bm{\theta}_k - \bm{\theta})^\top \nabla L(\bm{\theta}_k) \right]\\ &\leq \epsilon^2,
\end{split}
\end{align}
where $\epsilon > 0$ and $\delta \in (0,1]$ are precisions and $\bm{\theta} \in \mathcal{S}$. It would be sufficient to consider that the right-hand side of \eqref{evaluation}, $\mathrm{V}(K, \bm{\theta}) \leq \epsilon^2$, approximates a local minimizer of $L$. Since consideration of $\mathrm{V}(K, \bm{\theta}) \leq \epsilon^2$ leads to the lower bound of $K$ satisfying $\mathrm{V}(K, \bm{\theta}) \leq \epsilon^2$, we also consider $\delta \epsilon^2 \leq \mathrm{V}(K, \bm{\theta})$, which leads to the upper bound of $K$.

\section{Main Results}
\label{sec:3}
Let us suppose that (S1), (S2)\eqref{gradient}, and (S3) hold. Then, Algorithm \ref{algo:1} satisfies the following equation (see Lemma \ref{lem:key} for details). For all $\bm{\theta}\in\mathcal{S}$ and all $K \geq 1$,
\begin{align}\label{key0}
\begin{split}
&\mathrm{V}(K,\bm{\theta})\\
&=
\frac{1}{2 \alpha \tilde{\beta} K} 
\underbrace{\sum_{k=1}^K \tilde{\gamma}_k
\left\{ 
\mathbb{E}\left[\left\| \bm{\theta}_{k} - \bm{\theta} \right\|_{\mathsf{H}_k}^2\right]
-
\mathbb{E}\left[\left\| \bm{\theta}_{k+1} - \bm{\theta} \right\|_{\mathsf{H}_k}^2\right]
\right\}}_{\Gamma_K}\\
& 
+ \frac{\alpha}{2 \tilde{\beta}K} \underbrace{\sum_{k=1}^K \tilde{\gamma}_k \mathbb{E} \left[ \left\|\bm{\mathsf{d}}_k \right\|_{\mathsf{H}_k}^2 \right]}_{A_K}
+ \frac{\beta}{\tilde{\beta}K}
\underbrace{
\sum_{k=1}^K \mathbb{E} \left[ (\bm{\theta} - \bm{\theta}_k)^\top \bm{m}_{k-1} \right]}_{B_K},
\end{split}
\end{align} 
where $\tilde{\beta} := 1 - \beta$, $\tilde{\gamma}_k := 1 - \gamma^{k+1}$, and $\|\bm{x}\|_S^2 := \bm{x}^\top S \bm{x}$ ($S \in \mathbb{S}_{++}^d$, $\bm{x} \in\mathbb{R}^d$).

\subsection{Lower bound of $K$}
We provide the lower bound of $K$ satisfying $\mathrm{V}(K,\bm{\theta}) \leq \epsilon^2$ under the following conditions:
\begin{enumerate}
\item[(L1)] There exists a positive number $P$ such that $\mathbb{E}[\| \nabla L (\bm{\theta}_k) \|^2] \leq P^2$, where $\bm{\theta}_k = (\theta_{k,i})$ is the sequence generated by Algorithm \ref{algo:1}.
\item[(L2)] For all $\bm{\theta} = (\theta_i) \in \mathcal{S}$, there exists a positive number $\mathrm{Dist}(\bm{\theta})$ such that $\max_{i\in [d]} \sup \{ (\theta_{k,i} - \theta_i)^2 \colon k \in \mathbb{N} \} \leq \mathrm{Dist}(\bm{\theta})$.
\end{enumerate}

Condition (L1) was used in previous studies \cite[A2]{chen2019} and \cite[Theorem 2.2]{ada} to analyze deep learning optimizers for nonconvex optimization. Here, we use (L1) to provide the upper bounds of $\mathbb{E}[\|\bm{\mathsf{d}}_k\|_{\mathsf{H}_k}^2]$ and $\mathbb{E}[\|\bm{m}_k\|^2]$ (see Lemma \ref{lem:bdd} for details); i.e., 
\begin{align}\label{m}
\mathbb{E}\left[ \| \bm{m}_k \|^2 \right]
\leq \frac{\sigma^2}{b} + P^2
\text{ and } 
\mathbb{E}\left[\|\bm{\mathsf{d}}_k\|_{\mathsf{H}_k}^2 \right]
\leq \frac{\sigma^2 b^{-1} + P^2}{\tilde{\gamma} h_0^*},
\end{align}
where (S2)\eqref{sigma} and (A1) are assumed, $\tilde{\gamma} := 1 - \gamma$, and $h_0^* := \min_{i\in [d]} h_{0,i}$. This formulation implies that $A_K$ in \eqref{key0} is bounded above (see Theorems \ref{thm:main} and \ref{thm:1} for details).

Condition (L2) is assumed for both convex and nonconvex optimization, e.g., \cite[p.1574]{nem2009}, \cite[Theorem 4.1]{adam}, \cite[p.2]{reddi2018}, \cite[Theorem 4]{luo2019}, and \cite[Theorem 2.1]{ada}. Here, we use (L2) and (A2) to provide the upper bounds of $\Gamma_K$ and $B_K$ in \eqref{key0} (see Theorem \ref{thm:main} for details); i.e., 
\begin{align*}
\Gamma_K \leq d \mathrm{Dist}(\bm{\theta}) H
\text{ and } 
B_K \leq \sqrt{d \mathrm{Dist}(\bm{\theta})(\sigma^2 + P^2)} K,
\end{align*}
where (S2)\eqref{sigma} and (L1) are assumed and $H := \max_{i\in [d]} H_i$. Therefore, we have the following theorem (see the Appendix for detailed proof): 
 
\begin{thm}\label{thm:1}
Suppose that (S1)--(S3), (A1)--(A2), and (L1)--(L2) hold. Then, Algorithm \ref{algo:1} satisfies that, for all $\bm{\theta} \in \mathcal{S}$ and all $K \geq 1$,
\begin{align*}
\mathrm{V}(K,\bm{\theta}) 
\leq \frac{A}{K} + \frac{B}{b} + C + D,
\end{align*}
where $\tilde{\beta} := 1 - \beta$, $\tilde{\gamma} := 1 -\gamma$, $h_0^* := \min_{i\in [d]} h_{0,i}$,
\begin{align*}
&A = A(\alpha,\beta,\bm{\theta},d,H) = 
\frac{d \mathrm{Dist}(\bm{\theta}) H}{2 \alpha \tilde{\beta}},\\
&B = B(\alpha,\beta,\gamma,\sigma^2,h_0^*) = 
\frac{\sigma^2 \alpha}{2 \tilde{\beta} \tilde{\gamma}^2 h_0^*},\\
&C = C(\alpha,\beta,\gamma,P^2,h_0^*) =
\frac{P^2 \alpha}{2 \tilde{\beta} \tilde{\gamma}^2 h_0^*},\\
&D = D(\beta,\bm{\theta},d,\sigma^2, P^2) = 
\frac{\sqrt{d\mathrm{Dist}(\bm{\theta})(\sigma^2 + P^2)} \beta}{\tilde{\beta}}.
\end{align*}
Moreover, the lower bound $\underline{K}_{\epsilon}$ of the number of steps $K_\epsilon$ satisfying $\mathrm{V}(K,\bm{\theta}) \leq \epsilon^2$ 
for Algorithm \ref{algo:1} is such that, for all $b > B/\{\epsilon^2 - (C + D)\} \geq 0$,
\begin{align*}
\underline{K}_{\epsilon}(b) 
:= \frac{A b}{\{\epsilon^2 - (C + D)\}b - B} \leq K_\epsilon.
\end{align*}
The lower bound $\underline{K}_{\epsilon}$ is monotone decreasing and convex: 
\begin{align*}
\underline{K}_{\epsilon}^\star := 
\frac{A}{\epsilon^2 - (C+D)}
< \underline{K}_{\epsilon}(b)
\end{align*}
for all $b > B/\{\epsilon^2 - (C + D)\}$.
\end{thm}

The minimization of $\underline{K}_\epsilon (b) b$ is as follows (The proof of Theorem \ref{thm:1_1} is given in the Appendix):

\begin{thm}\label{thm:1_1}
Suppose that the assumptions in Theorem \ref{thm:1} hold and consider Algorithm \ref{algo:1}. Then, there exists 
\begin{align*}
b_* := \frac{2B}{\epsilon^2 - (C+D)} 
\end{align*}
such that $b_*$ minimizes a convex function $\underline{K}_{\epsilon}(b) b$; i.e., for all $b > B/\{\epsilon^2 - (C + D)\}$,
\begin{align*}
\frac{4AB}{\{\epsilon^2 - (C+D)\}^2} 
&= \underline{K}_{\epsilon}(b_*) b_* 
\leq \underline{K}_{\epsilon}(b) b\\ 
&= 
\frac{Ab^2}{\{\epsilon^2 - (C+D)\}b - B}.
\end{align*}
\end{thm}

\subsection{Upper bound of $K$}
We provide the upper bound of $K$ satisfying $\delta \epsilon^2 \leq \mathrm{V}(K,\bm{\theta})$ under the following conditions: 
\begin{enumerate}
\item[(U1)] There exists $c_1 \in [0,1]$ such that $\mathbb{E}[\|\bm{\mathsf{d}}_k\|_{\mathsf{H}_k}^2] \geq c_1 \sigma^2/b$.
\item[(U2)] There exist $\bm{\theta}^* \in \mathrm{VI}(\mathcal{S},\nabla L)$ and $c_2 \geq 0$ such that 
$\mathbb{E}[(\bm{\theta}^* - \bm{\theta}_k)^\top \bm{m}_{k-1}] \geq - c_2 (\sigma^2/b + P^2)$.
\item[(U3)] For $\bm{\theta}^* \in \mathrm{VI}(\mathcal{S},\nabla L)$, there exists a positive number $X(\bm{\theta}^*)$ such that, for all $k \geq 1$, $\tilde{\gamma} \mathbb{E}[\|\bm{\theta}_1 - \bm{\theta}^*\|_{\mathsf{H}_1}^2] - \mathbb{E}[\|\bm{\theta}_{k+1} - \bm{\theta}^*\|_{\mathsf{H}_k}^2] \geq X(\bm{\theta}^*)$, where $\tilde{\gamma} := 1 - \gamma$.
\end{enumerate}
We consider the case where $\mathsf{H}_k$ is the identity matrix needed to justify (U1). The definition of $\bm{m}_k := \beta \bm{m}_{k-1} + (1-\beta) \nabla L_{B_k}(\bm{\theta}_k)$ implies that $\bm{m}_k$ depends on $\nabla L_{B_k}(\bm{\theta}_k)$. Here, we assume the following condition, which is stronger than (S2)\eqref{sigma}: there exists $c_1 \in [0,1]$ such that 
\begin{align}\label{sigma1}
c_1 \sigma^2 \leq \mathbb{E}_{\xi_k}\left[ \left\|\mathsf{G}_{\xi_k}(\bm{\theta}_k) - 
\nabla L(\bm{\theta}_k) \right\|^2 \right] \leq \sigma^2.
\end{align}
Then, (S2)\eqref{gradient} and \eqref{sigma1} ensure that 
\begin{align*}
\mathbb{E}\left[ \| \nabla L_{B_k}(\bm{\theta}_k)\|^2 \big| \bm{\theta}_k \right]
&\geq 
\mathbb{E}\left[ \| \nabla L_{B_k}(\bm{\theta}_k) - \nabla L (\bm{\theta}_k)\|^2 \big| \bm{\theta}_k \right]\\
&\geq \frac{c_1 \sigma^2}{b}.
\end{align*}
From $\tilde{\gamma}_k = 1 - \gamma^{k+1} \leq 1$, we have that $\|\bm{\mathsf{d}}_k\|^2 = \tilde{\gamma}_k^{-2} \|\bm{m}_k\|^2 \geq \|\bm{m}_k\|^2$ ($k\in\mathbb{N}$). Accordingly, the lower bound of $\|\bm{\mathsf{d}}_k\|^2$ depends on $c_1 \sigma^2/b$.

Algorithm \ref{algo:1} for nonconvex optimization satisfies that there exist positive numbers $C_1$ and $C_2$ such that, for all $\bm{\theta} \in \mathcal{S}$, 
\begin{align}\label{iiduka}
\begin{split}
&\liminf_{k \to + \infty} \mathbb{E}\left[ (\bm{\theta}_k - \bm{\theta})^\top \nabla L (\bm{\theta}_k) \right]
\leq C_1 \alpha + C_2 \beta,\\
&\mathrm{V}(K,\bm{\theta})
\leq \mathcal{O}\left(\frac{1}{K} \right) + C_1 \alpha + C_2 \beta
\end{split}
\end{align}
(see \cite[Theorem 1]{iiduka2021}).
The sequence $(\bm{\theta}_k)_{k\in\mathbb{N}}$ generated by Algorithm \ref{algo:1} can thus approximate a local minimizer $\bm{\theta}^* \in \mathrm{VI}(\mathcal{S},\nabla L)$. 
Meanwhile, we have that 
\begin{align*}
&(\bm{\theta}^* - \bm{\theta}_k)^\top \bm{m}_{k-1}\\
&= 
\frac{1}{2} 
\left\{
\|\bm{\theta}^* - \bm{\theta}_k\|^2 + \|\bm{m}_{k-1}\|^2
- \|(\bm{\theta}^* - \bm{\theta}_k) - \bm{m}_{k-1}\|^2
\right\}\\
&\geq 
-\frac{1}{2} \|(\bm{\theta}^* - \bm{\theta}_k) - \bm{m}_{k-1}\|^2.
\end{align*}
Accordingly, for a sufficiently large $k$, the lower bound of $(\bm{\theta}^* - \bm{\theta}_k)^\top \bm{m}_{k-1}$ depends on 
$- \mathbb{E}[\|\bm{m}_{k-1}\|^2] \geq - (\sigma^2/b + P^2)$,
where (S2)\eqref{sigma} and (L1) are assumed (see \eqref{m}).
Hence, we assume (U2) for Algorithm \ref{algo:1}. 

Condition (U3) implies that $\mathbb{E}[\|\bm{\theta}_{k+1} - \bm{\theta}^*\|_{\mathsf{H}_k}^2] < \mathbb{E}[\|\bm{\theta}_1 - \bm{\theta}^*\|_{\mathsf{H}_1}^2] < + \infty$, which is stronger than (L2), and that $\bm{\theta}_{k+1}$ approximates more appropriately $\bm{\theta}^*$ than $\bm{\theta}_1$. Since sequence $(\bm{\theta}_k)_{k\in\mathbb{N}}$ generated by Algorithm \ref{algo:1} can approximate a local minimizer $\bm{\theta}^* \in \mathrm{VI}(\mathcal{S},\nabla L)$ (see \eqref{iiduka}), (U3) is also adequate for consideration of the upper bound of $K$. 

\begin{thm}\label{thm:2}
Suppose that (S1)--(S3), (A1), (L1), and (U1)--(U3) hold. Then, Algorithm \ref{algo:1} satisfies that, for all $K \geq 1$,
\begin{align*}
\mathrm{V}(K,\bm{\theta}^*) \geq 
\frac{E}{K}
+ 
\frac{F}{b}
- 
G,
\end{align*}
where $\tilde{\beta} := 1 - \beta$, $\tilde{\gamma} := 1 - \gamma$,
\begin{align*}
&E = E(\alpha,\beta,\bm{\theta}^*) = \frac{X (\bm{\theta}^*)}{2 \alpha \tilde{\beta}},\\
&F = F(\alpha,\beta,\gamma,\sigma^2, c_1, c_2) 
= \frac{\sigma^2 (c_1 \alpha \tilde{\gamma} - 2 c_2 \beta)}{2 \tilde{\beta}}, \\
&G = G (\beta, c_2, P^2) = \frac{c_2 \beta P^2}{\tilde{\beta}}.
\end{align*}
Moreover, the upper bound $\overline{K}_{\epsilon,\delta}$ of the number of steps $K_{\epsilon,\delta}$ satisfying $\delta \epsilon^2 \leq \mathrm{V}(K,\bm{\theta}^*)$ 
for Algorithm \ref{algo:1} is such that, for all $b > F/(\delta \epsilon^2 + G) \geq 0$,
\begin{align*}
\overline{K}_{\epsilon,\delta}(b) 
:= \frac{E b}{(\delta \epsilon^2 + G) b - F} \geq K_{\epsilon,\delta}.
\end{align*}
The upper bound $\overline{K}_{\epsilon,\delta}$ is monotone decreasing and convex: 
\begin{align*}
\overline{K}_{\epsilon,\delta}^\star := 
\frac{E}{\delta \epsilon^2 + G}
< \overline{K}_{\epsilon,\delta}(b)
\end{align*}
for all $b > F/(\delta \epsilon^2 + G)$.
\end{thm}

The minimization of $\overline{K}_{\epsilon,\delta} (b) b$ is as follows (The proof of Theorem \ref{thm:2_1} is given in the Appendix):

\begin{thm}\label{thm:2_1}
Suppose that the assumptions in Theorem \ref{thm:2} hold and consider Algorithm \ref{algo:1}. Then, there exists 
\begin{align*}
b^* := \frac{2F}{\delta \epsilon^2 + G} 
\end{align*}
such that $b^*$ minimizes a convex function $\overline{K}_{\epsilon,\delta}(b) b$; i.e., for all $b > F/(\delta \epsilon^2 + G)$,
\begin{align*}
\frac{4EF}{(\delta \epsilon^2 + G)^2} 
= 
\overline{K}_{\epsilon,\delta}(b^*) b^* 
\leq \overline{K}_{\epsilon,\delta}(b) b
= 
\frac{E b^2}{(\delta \epsilon^2 + G) b - F}.
\end{align*}
\end{thm}

\begin{figure*}[htbp]
\begin{tabular}{cc}
\begin{minipage}[t]{0.45\hsize}
\centering
\includegraphics[width=1\textwidth]{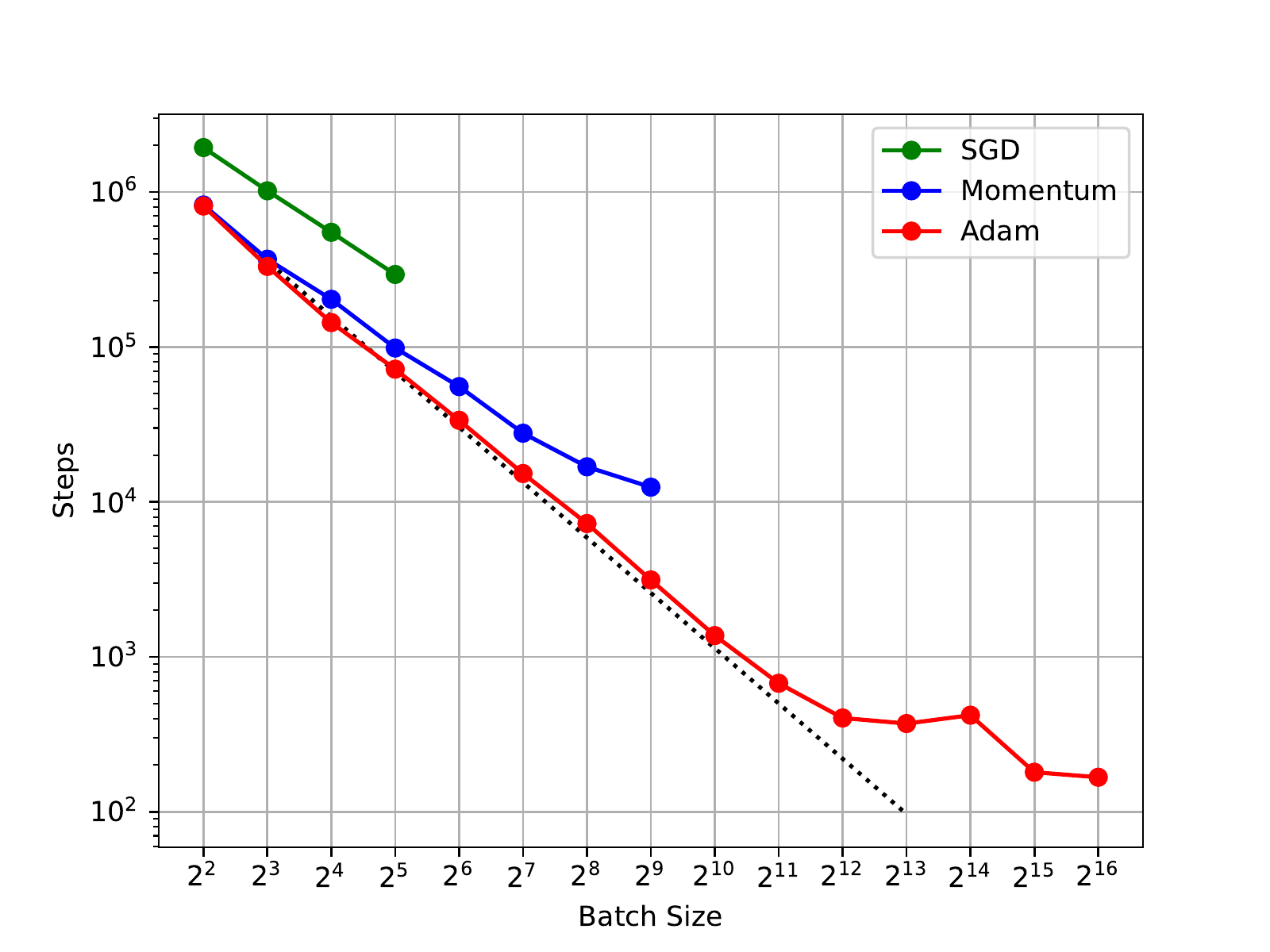}
\caption{Number of steps for SGD, Momentum, and Adam versus batch size needed to train ResNet-20 on CIFAR-10. There is an initial period of perfect scaling (indicated by dashed line) such that the number of steps $K$ needed for $L(\bm{\theta}_K) \leq 10^{-1}$ is inversely proportional to batch size $b$. However, Adam has critical batch size $b^\star = 2^{11}$.}
\label{fig1}
\end{minipage} &
\begin{minipage}[t]{0.45\hsize}
\centering
\includegraphics[width=1\textwidth]{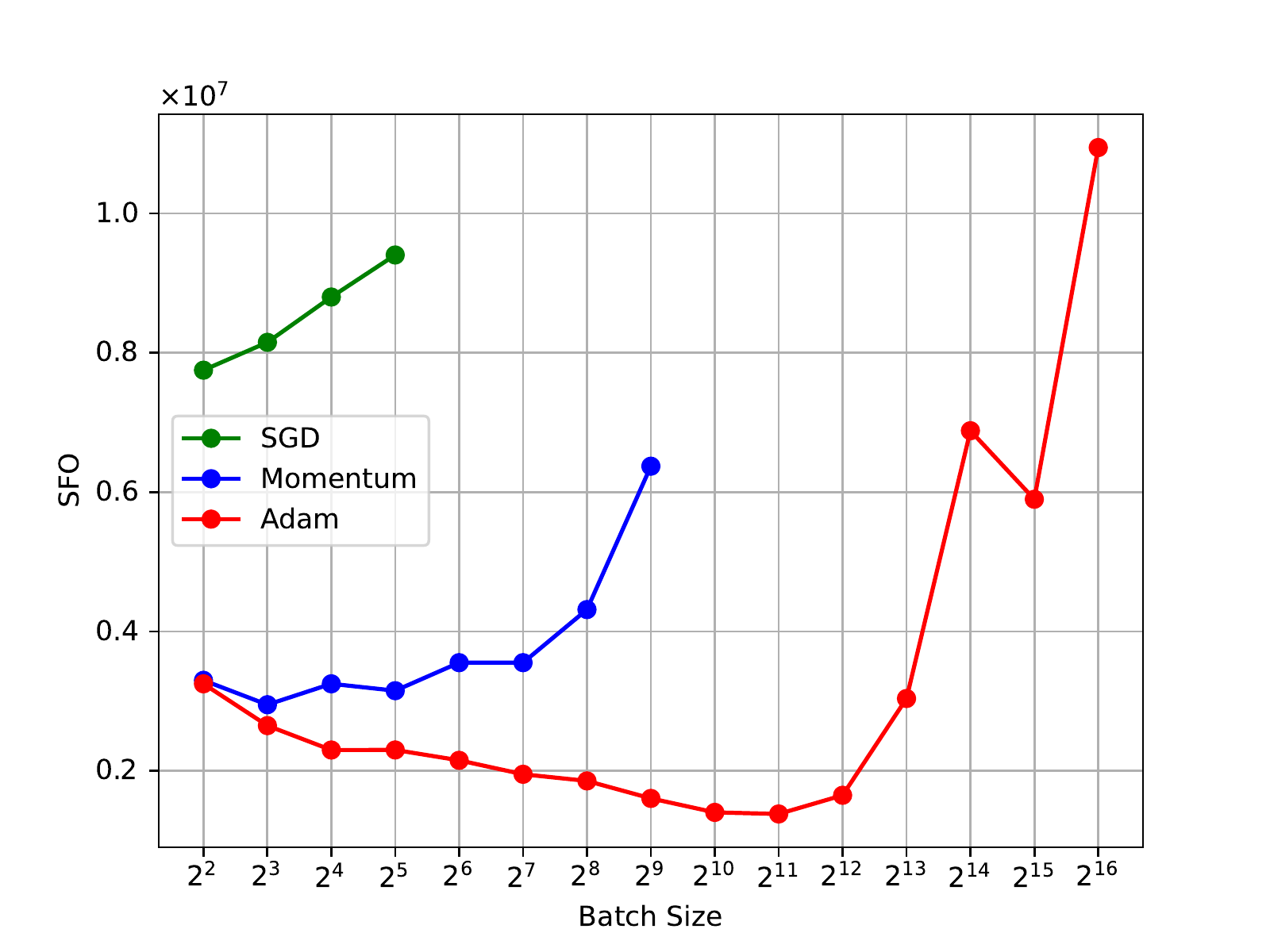}
\caption{SFO complexities for SGD, Momentum, and Adam versus batch size needed to train ResNet-20 on CIFAR-10. SFO complexity of Adam is minimum at critical batch size $b^\star = 2^{11}$, whereas SFO complexity for SGD and Momentum tends to increase with batch size.}
\label{fig2}
\end{minipage}
 \end{tabular}
\end{figure*} 

\begin{table*}[htbp]
\centering
\caption{Elapsed time and training accuracy of SGD when $L(\bm{\theta}_K) \leq 10^{-1}$ for training ResNet-20 on CIFAR-10}\label{table:sgd}
\begin{tabular}{lllllllll}
\bottomrule
\multicolumn{9}{c}{SGD} \\
Batch Size 
& $2^2$ 
& $2^3$ 
& $2^4$ 
& $2^5$ 
& $2^6$ 
& $2^7$ 
& $2^8$ 
& $2^9$ \\
\hline
Time (s) 
& 16983.64 
& 9103.76 
& 6176.19 
& 3759.25 
& --- 
& --- 
& --- 
& --- \\
Accuracy (\%) 
& 96.75 
& 96.69
& 96.66 
& 96.88 
& --- 
& --- 
& --- 
& --- \\
\bottomrule
\end{tabular}
\end{table*}

\begin{table*}[htbp]
\centering
\caption{Elapsed time and training accuracy of Momentum when $L(\bm{\theta}_K) \leq 10^{-1}$ for training ResNet-20 on CIFAR-10}\label{table:momentum}
\begin{tabular}{lllllllll}
\bottomrule
\multicolumn{9}{c}{Momentum} \\
Batch Size 
& $2^2$ 
& $2^3$ 
& $2^4$ 
& $2^5$ 
& $2^6$ 
& $2^7$ 
& $2^8$ 
& $2^9$ \\
\hline
Time (s) 
& 7978.90 
& 3837.72 
& 2520.82 
& 1458.70 
& 887.01 
& 678.66 
& 625.10 
& 866.65 \\
Accuracy (\%) 
& 96.49 
& 96.79
& 96.51 
& 96.72 
& 96.70 
& 96.94 
& 96.94 
& 98.34 \\
\bottomrule
\end{tabular}
\end{table*}

\begin{table*}[htbp]
\centering
\caption{Elapsed time and training accuracy of Adam when $L(\bm{\theta}_K) \leq 10^{-1}$ for training ResNet-20 on CIFAR-10}\label{table:adam1}
\begin{tabular}{lllllllll}
\bottomrule
\multicolumn{9}{c}{Adam} \\
Batch Size 
& $2^2$ 
& $2^3$ 
& $2^4$ 
& $2^5$ 
& $2^6$ 
& $2^7$ 
& $2^8$ 
& $2^9$ \\
\hline
Time (s) 
& 10601.78 
& 4405.73
& 2410.28
& 1314.01 
& 617.14 
& 487.75 
& 281.74 
& 225.03 \\
Accuracy (\%) 
& 96.46 
& 96.38
& 96.65 
& 96.53 
& 96.43 
& 96.68 
& 96.58 
& 96.72 \\
\bottomrule
\end{tabular}
\end{table*}

\begin{table*}[htbp]
\centering
\caption{Elapsed time and training accuracy of Adam when $L(\bm{\theta}_K) \leq 10^{-1}$ for training ResNet-20 on CIFAR-10}\label{table:adam2}
\begin{tabular}{llllllll}
\bottomrule
\multicolumn{8}{c}{Adam} \\
Batch Size 
& $2^{10}$ 
& $2^{11}$ 
& $2^{12}$ 
& $2^{13}$ 
& $2^{14}$ 
& $2^{15}$ 
& $2^{16}$ \\
\hline
Time (s) 
& 197.78 
& 195.40
& 233.70
& 349.81 
& 691.04 
& 644.19 
& 1148.68\\
Accuracy (\%) 
& 96.74 
& 97.21
& 97.54 
& 97.75 
& 97.51 
& 99.05 
& 99.03 \\
\bottomrule
\end{tabular}
\end{table*}

\subsection{Discussion}
To guarantee that $\underline{K}_{\epsilon} (b) b$ in Theorem \ref{thm:1_1} and $\overline{K}_{\epsilon,\delta} (b) b$ in Theorem \ref{thm:2_1} fit the actual SFO complexity $Kb$, we need to satisfy \eqref{condition}; that is, 
\begin{align*}
&\underbrace{(C+D) F}_{\approx (M_1 \alpha + M_2)(M_3 \alpha - M_4)} + 
\underbrace{BG}_{\approx M_5 \alpha} \approx (F - \delta B) \epsilon^2 
\text{ and}\\
&\underbrace{(C+D)}_{\approx M_1 \alpha + M_2} \sqrt{EF} 
+ G \sqrt{AB} \approx (\sqrt{EF}- \delta \sqrt{AB})\epsilon^2.
\end{align*}
Hence, if precisions $\epsilon$ and $\delta$ are small, then using a sufficiently small learning rate $\alpha$ could be used to approximate $K(b)b$. Indeed, small learning rates such as $10^{-2}$ and $10^{-3}$ have been practically used in previous numerical validations \cite{shallue2019}. The demonstration of the existence of $b_*$ and $b^*$ leads to the existence of the actual critical batch size $b^\star \approx b_* \approx b^*$.

\section{Numerical Examples}
\label{sec:4}
We evaluated the performances of SGD, Momentum, and Adam with different batch sizes for training ResNet-20 on the CIFAR-10 dataset with $n = 50000$. The metrics were the number of steps $K$ and the SFO complexity $Kb$ satisfying $L(\bm{\theta}_K) \leq 10^{-1}$, where $\bm{\theta}_K$ is generated for each of SGD, Momentum, and Adam using batch size $b$. The stopping condition was $200$ epochs. The experimental environment consisted of two Intel(R) Xeon(R) Gold 6148 2.4-GHz CPUs with 20 cores each, a 16-GB NVIDIA Tesla V100 900-Gbps GPU, and the Red Hat Enterprise Linux 7.6 OS. The code was written in Python 3.8.2 using the NumPy 1.17.3 and PyTorch 1.3.0 packages. A constant learning rate ($\alpha = 10^{-3}$) was commonly used. SGD used the identity matrix $\mathsf{H}_k$ and $\beta = \gamma = 0$, and Momentum used the identity matrix $\mathsf{H}_k$, $\beta = 0.9$, and $\gamma = 0$. Adam used $\mathsf{H}_k$ defined by Table \ref{table:ex}, $\beta = \gamma = 0.9$, and $\eta = \zeta = 0.999$ \cite{adam}.

Figure \ref{fig1} shows the number of steps for SGD, Momentum, and Adam versus the batch size. For SGD and Momentum, the number of steps $K$ needed for $L(\bm{\theta}_K) \leq 10^{-1}$ initially decreased. However, SGD with $b \geq 2^6$ and Momentum with $b \geq 2^{10}$ did not satisfy $L(\bm{\theta}_K) \leq 10^{-1}$ before the stopping condition was reached. Adam had an initial period of perfect scaling (indicated by dashed line) such that the number of steps $K$ needed for $L(\bm{\theta}_K) \leq 10^{-1}$ was inversely proportional to batch size $b$, and critical batch size $b^\star = 2^{11}$ such that $K$ was not inversely proportional to the batch size beyond $b^\star$; i.e., there were diminishing returns. 

Figure \ref{fig2} plots the SFO complexities for SGD, Momentum, and Adam versus the batch size. For SGD, SFO complexity was minimum at $b = 2^2$; for Momentum, it was minimum at $b = 2^3$. For Adam, SFO complexity was minimum at the critical batch size $b^\star = 2^{11}$. In particular, we determined that the SFO complexity of Adam with $b = 2^{10}$ was about $1404928$, that of Adam with $b^\star = 2^{11}$ was about $1382400$, and that of Adam with $b = 2^{12}$ was about $1650688$.

Tables \ref{table:sgd}, \ref{table:momentum}, \ref{table:adam1}, and \ref{table:adam2} show the elapsed time and the training accuracy for SGD, Momentum, and Adam when $L(\bm{\theta}_K) \leq 10^{-1}$ was achieved. Table \ref{table:sgd} shows that the elapsed time for SGD decreased with an increase in batch size. Table \ref{table:momentum} shows that the elapsed time for Momentum monotonically decreased for $b \leq 2^8$ and increased for $b = 2^9$ compared with that for $b = 2^8$. This is because the SFO complexity for Momentum for $b = 2^9$ was substantially higher than that for $b = 2^8$, as shown in Figure \ref{fig2}. Tables \ref{table:adam1} and \ref{table:adam2} show that the elapsed time for Adam monotonically decreased for $b \leq 2^{11}$ and that the elapsed time for critical batch size $b^\star = 2^{11}$ was the shortest. The elapsed time for $b \geq 2^{12}$ increased with the SFO complexity, as shown in Figure \ref{fig2}.

We also investigated the behaviors of SGD, Momentum, and Adam for a decaying learning rate ($\alpha_k = 10^{-3}/\sqrt{k}$) and a decaying learning rate ($\alpha_k$ defined by $\alpha_k = 10^{-2} - 10^{-2} (1-\alpha) k T^{-1}$ ($k \leq T$) or $\alpha_k = 10^{-2} \alpha$ ($k > T$)), where the decay factor was $\alpha = 10^{-3}$ and the number of training steps was $T = 10^4$ ($\alpha$ and $T$ were set on the basis of the results of a previous study \cite[Figure 14]{shallue2019}). None of the optimizers achieved $L(\bm{\theta}_K) \leq 10^{-1}$ before the stopping condition ($200$ epochs) was reached.

\section{Conclusion}
We have clarified the lower and upper bounds of the number of steps needed for nonconvex optimization of adaptive methods and have shown that there exist critical batch sizes $b_*$ and $b^*$ in the sense of minimizing the lower and upper bounds of the stochastic first-order oracle (SFO) complexities of adaptive methods. Previous numerical evaluations \cite{shallue2019,zhang2019} showed that the actual critical batch size $b^\star$ minimizes SFO complexity. Hence, if the lower and upper bounds of the SFO complexity fit the actual SFO complexity, then the existence of $b_*$ and $b^*$ guarantee the existence of the actual critical batch size $b^\star$. We also demonstrated that it would be sufficient to use small learning rates for an adaptive method to satisfy that the lower and upper bounds of SFO complexity fit the actual SFO complexity. Furthermore, we provided numerical examples supporting our theoretical results. They showed that Adam with a frequently used small learning rate ($10^{-3}$) had a critical batch size that minimizes SFO complexity.

\appendix
\section{Appendix}
\label{app:1}
Unless stated otherwise, all relationships between random variables are supported to hold almost surely. Let $S \in \mathbb{S}_{++}^d$. The $S$-inner product of $\mathbb{R}^d$ is defined for all $\bm{x}, \bm{y} \in \mathbb{R}^d$ by $\langle \bm{x},\bm{y} \rangle_S := \langle \bm{x}, S \bm{y} \rangle = \bm{x}^\top (S \bm{y})$, and the $S$-norm is defined by $\|\bm{x}\|_S := \sqrt{\langle \bm{x}, S \bm{x} \rangle}$.

\subsection{Lemmas and Theorems}
\begin{lem}\label{lem:1}
Suppose that (S1), (S2)\eqref{gradient}, and (S3) hold. Then, Algorithm \ref{algo:1} satisfies the following: for all $k\in\mathbb{N}$ and all $\bm{\theta} \in \mathcal{S}$,
\begin{align*}
\mathbb{E}\left[\| \bm{\theta}_{k+1} - \bm{\theta} \|_{\mathsf{H}_k}^2 \right]
&=
\mathbb{E}\left[\| \bm{\theta}_{k} - \bm{\theta} \|_{\mathsf{H}_k}^2 \right]
+ \alpha^2 \mathbb{E}\left[\| \bm{\mathsf{d}}_k \|_{\mathsf{H}_k}^2 \right]\\
&\quad + 2 \alpha \left\{
\frac{\beta}{{\tilde{\gamma}}_k} 
\mathbb{E}\left[ (\bm{\theta} - \bm{\theta}_k)^\top \bm{m}_{k-1} \right] 
+\frac{\tilde{\beta}}{{\tilde{\gamma}}_k} 
\mathbb{E}\left[ (\bm{\theta} - \bm{\theta}_k)^\top \nabla L (\bm{\theta}_k) \right]
\right\},
\end{align*}
where $\tilde{\beta} := 1 - \beta$ and ${\tilde{\gamma}}_k := 1 - {\gamma}^{k+1}$.
\end{lem}

\begin{proof}
Let $\bm{\theta} \in \mathcal{S}$ and $k\in\mathbb{N}$. The definition of $\bm{\theta}_{k+1}$ implies that 
\begin{align*}
\| \bm{\theta}_{k+1} - \bm{\theta} \|_{\mathsf{H}_k}^2
= 
\| \bm{\theta}_{k} - \bm{\theta} \|_{\mathsf{H}_k}^2
+ 2 \alpha \langle \bm{\theta}_{k} - \bm{\theta}, \bm{\mathsf{d}}_k \rangle_{\mathsf{H}_k}
+ \alpha^2 \|\bm{\mathsf{d}}_k\|_{\mathsf{H}_k}^2.
\end{align*}
Moreover, the definitions of $\bm{\mathsf{d}}_k$, $\bm{m}_k$, and $\hat{\bm{m}}_k$ ensure that 
\begin{align*}
\left\langle \bm{\theta}_k - \bm{\theta}, \bm{\mathsf{d}}_k \right\rangle_{\mathsf{H}_k}
=
\frac{1}{{\tilde{\gamma}}_k}
(\bm{\theta} - \bm{\theta}_k)^\top \bm{m}_k 
=
\frac{\beta}{{\tilde{\gamma}}_k} 
(\bm{\theta} - \bm{\theta}_k)^\top \bm{m}_{k-1} 
+
\frac{\tilde{\beta}}{{\tilde{\gamma}}_k} 
(\bm{\theta} - \bm{\theta}_k)^\top \nabla L_{B_k}(\bm{\theta}_k).
\end{align*}
Hence, 
\begin{align}\label{ineq:004}
\begin{split}
\left\|\bm{\theta}_{k+1} - \bm{\theta} \right\|_{\mathsf{H}_k}^2
&=
\left\| \bm{\theta}_k -\bm{\theta} \right\|_{\mathsf{H}_k}^2
+ \alpha^2 \left\| \bm{\mathsf{d}}_k \right\|_{\mathsf{H}_k}^2\\
&\quad + 2 \alpha \left\{
\frac{\beta}{{\tilde{\gamma}}_k} 
(\bm{\theta} - \bm{\theta}_k)^\top \bm{m}_{k-1} 
+ \frac{\tilde{\beta}}{{\tilde{\gamma}}_k} 
(\bm{\theta} - \bm{\theta}_k)^\top \nabla L_{B_k} (\bm{\theta}_k) 
\right\}.
\end{split}
\end{align}
Conditions \eqref{gradient} and (S3) guarantee that
\begin{align*}
\mathbb{E}\left[ \mathbb{E} \left[(\bm{\theta} - \bm{\theta}_k)^\top \nabla L_{B_k} (\bm{\theta}_k) \Big| \bm{\theta}_k \right] \right]
&=
\mathbb{E} \left[(\bm{\theta} - \bm{\theta}_k)^\top 
\mathbb{E} \left[\nabla L_{B_k} (\bm{\theta}_k) \Big| \bm{\theta}_k \right] \right]\\
&=
\mathbb{E} \left[(\bm{\theta} - \bm{\theta}_k)^\top 
\nabla L (\bm{\theta}_k) \right].
\end{align*}
Therefore, the lemma follows by taking the expectation with respect to $\xi_k$ on both sides of \eqref{ineq:004}. This completes the proof.
\end{proof}

Lemma \ref{lem:1} leads to the following:

\begin{lem}\label{lem:key}
Suppose that the assumptions in Lemma \ref{lem:1} hold and consider Algorithm \ref{algo:1}. Let $V_k (\bm{\theta}) := \mathbb{E}[(\bm{\theta}_k - \bm{\theta})^\top \nabla L (\bm{\theta}_k)]$ for all $\bm{\theta}\in \mathcal{S}$ and all $k\in\mathbb{N}$. Then, for all $\bm{\theta}\in\mathcal{S}$ and all $K \geq 1$,
\begin{align}\label{key}
\begin{split}
{\sum_{k=1}^K V_k (\bm{\theta})}
&=
\frac{1}{2 \alpha \tilde{\beta}} 
\underbrace{\sum_{k=1}^K \tilde{\gamma}_k
\left\{ 
\mathbb{E}\left[\left\| \bm{\theta}_{k} - \bm{\theta} \right\|_{\mathsf{H}_k}^2\right]
-
\mathbb{E}\left[\left\| \bm{\theta}_{k+1} - \bm{\theta} \right\|_{\mathsf{H}_k}^2\right]
\right\}}_{\Gamma_K}\\
&\quad 
+ \frac{\alpha}{2 \tilde{\beta}} \underbrace{\sum_{k=1}^K \tilde{\gamma}_k \mathbb{E} \left[ \left\|\bm{\mathsf{d}}_k \right\|_{\mathsf{H}_k}^2 \right]}_{A_K}
+ \frac{\beta}{\tilde{\beta}}
\underbrace{
\sum_{k=1}^K \mathbb{E} \left[ (\bm{\theta} - \bm{\theta}_k)^\top \bm{m}_{k-1} \right]}_{B_K}.
\end{split}
\end{align} 
\end{lem}

\begin{proof}
Let $\bm{\theta} \in \mathcal{S}$. Lemma \ref{lem:1} guarantees that, for all $k\in\mathbb{N}$,
\begin{align*}
\frac{2 \alpha \tilde{\beta}}{\tilde{\gamma}_k} V_k (\bm{\theta})
&=
\mathbb{E}\left[\left\| \bm{\theta}_{k} - \bm{\theta} \right\|_{\mathsf{H}_k}^2\right]
-
\mathbb{E}\left[\left\| \bm{\theta}_{k+1} - \bm{\theta} \right\|_{\mathsf{H}_k}^2\right]
+ \alpha^2 \mathbb{E} \left[ \left\|\bm{\mathsf{d}}_k \right\|_{\mathsf{H}_k}^2 \right]\\
&\quad + \frac{2 \alpha \beta}{\tilde{\gamma}_k} \mathbb{E} \left[ 
(\bm{\theta} - \bm{\theta}_k)^\top \bm{m}_{k-1} \right].
\end{align*}
Hence, we have that 
\begin{align*}
V_k (\bm{\theta})
&=
\frac{\tilde{\gamma}_k}{2 \alpha \tilde{\beta}}
\left\{
\mathbb{E}\left[\left\| \bm{\theta}_{k} - \bm{\theta} \right\|_{\mathsf{H}_k}^2\right]
-
\mathbb{E}\left[\left\| \bm{\theta}_{k+1} - \bm{\theta} \right\|_{\mathsf{H}_k}^2\right]
\right\}\\
&\quad +
\frac{\alpha \tilde{\gamma}_k}{2 \tilde{\beta}} \mathbb{E} \left[ \left\|\bm{\mathsf{d}}_k \right\|_{\mathsf{H}_k}^2 \right]
+ \frac{\beta}{\tilde{\beta}} \mathbb{E} \left[ 
(\bm{\theta} - \bm{\theta}_k)^\top \bm{m}_{k-1} \right].
\end{align*}
Summing the above inequality from $k=1$ to $K \geq 1$ implies the assertion in Lemma \ref{lem:key}. This completes the proof.
\end{proof}

\begin{lem}\label{lem:bdd}
Algorithm \ref{algo:1} satisfies that, under (S2)\eqref{gradient}, \eqref{sigma}, and (L1), for all $k\in\mathbb{N}$,
\begin{align*}
\mathbb{E}\left[ \|\bm{m}_k\|^2 \right] 
\leq 
\frac{\sigma^2}{b} + P^2.
\end{align*}
Under (A1) and (L1), for all $k\in\mathbb{N}$,
\begin{align*}
\mathbb{E}\left[ \|\bm{\mathsf{d}}_k\|_{\mathsf{H}_k}^2 \right] 
\leq 
\frac{1}{(1-{\gamma})^2 h_0^*} \left( \frac{\sigma^2}{b} + P^2 \right),
\end{align*}
where $h_0^* := \min_{i\in [d]} h_{0,i}$.
\end{lem}

\begin{proof}
Let $k\in\mathbb{N}$. From (S2)\eqref{gradient}, we have that
\begin{align*}
\mathbb{E} \left[\| \nabla L_{B_k} (\bm{\theta}_{k}) \|^2
\big| \bm{\theta}_k
\right]
&=
\mathbb{E} \left[\| \nabla L_{B_k} (\bm{\theta}_{k}) 
- \nabla L (\bm{\theta}_{k}) + \nabla L (\bm{\theta}_{k}) \|^2
\big| \bm{\theta}_k
\right]\\
&=
\mathbb{E} \left[\| \nabla L_{B_k} (\bm{\theta}_{k}) 
- \nabla L (\bm{\theta}_{k}) \|^2 \big| \bm{\theta}_k
\right]
+ 
\mathbb{E} \left[\| \nabla L (\bm{\theta}_{k}) \|^2 \big| \bm{\theta}_k
\right]\\
&\quad + 2 
\mathbb{E} \left[ 
(\nabla L_{B_k} (\bm{\theta}_{k}) 
- \nabla L (\bm{\theta}_{k}))^\top \nabla L (\bm{\theta}_{k})
\Big| \bm{\theta}_k \right]\\
&= 
\mathbb{E} \left[\| \nabla L_{B_k} (\bm{\theta}_{k}) 
- \nabla L (\bm{\theta}_{k}) \|^2 \big| \bm{\theta}_k
\right]
+ 
\| \nabla L (\bm{\theta}_{k}) \|^2,
\end{align*} 
which, together with (S2)\eqref{sigma} and (L1), implies that 
\begin{align}\label{A3}
\mathbb{E} \left[\| \nabla L_{B_k} (\bm{\theta}_{k}) \|^2
\right]
\leq 
\frac{\sigma^2}{b} + P^2.
\end{align}
The convexity of $\|\cdot\|^2$, together with the definition of $\bm{m}_k$ and \eqref{A3}, guarantees that, for all $k\in\mathbb{N}$,
\begin{align*}
\mathbb{E}\left[ \|\bm{m}_k\|^2 \right]
&\leq \beta \mathbb{E}\left[ \|\bm{m}_{k-1} \|^2 \right] + 
(1-\beta) \mathbb{E}\left[ \|\nabla L_{B_k} (\bm{\theta}_k) \|^2 \right]\\
&\leq 
\beta \mathbb{E} \left[ \|\bm{m}_{k-1} \|^2 \right] + (1-\beta) 
\left( \frac{\sigma^2}{b} + P^2 \right).
\end{align*}
Induction thus ensures that, for all $k\in\mathbb{N}$,
\begin{align}\label{induction}
\mathbb{E} \left[\|\bm{m}_k \|^2 \right] \leq 
\max \left\{ \|\bm{m}_{-1}\|^2, \frac{\sigma^2}{b} + P^2 \right\} 
= \frac{\sigma^2}{b} + P^2,
\end{align}
where $\bm{m}_{-1} = \bm{0}$. For $k\in\mathbb{N}$, $\mathsf{H}_k \in \mathbb{S}_{++}^d$ guarantees the existence of a unique matrix $\overline{\mathsf{H}}_k \in \mathbb{S}_{++}^d$ such that $\mathsf{H}_k = \overline{\mathsf{H}}_k^2$ \cite[Theorem 7.2.6]{horn}. We have that, for all $\bm{x}\in\mathbb{R}^d$, $\|\bm{x}\|_{\mathsf{H}_k}^2 = \| \overline{\mathsf{H}}_k \bm{x} \|^2$. Accordingly, the definitions of $\bm{\mathsf{d}}_k$ and $\hat{\bm{m}}_k$ imply that, for all $k\in\mathbb{N}$, 
\begin{align*}
\mathbb{E} \left[ \| \bm{\mathsf{d}}_k \|_{\mathsf{H}_k}^2 \right]
= 
\mathbb{E} \left[ \left\| \overline{\mathsf{H}}_k^{-1} \mathsf{H}_k\bm{\mathsf{d}}_k \right\|^2 \right]
\leq 
\frac{1}{{\tilde{\gamma}}_k^2} \mathbb{E} \left[ \left\| \overline{\mathsf{H}}_k^{-1} \right\|^2 \|\bm{m}_k \|^2 \right]
\leq 
\frac{1}{(1 - \gamma)^2} \mathbb{E} \left[ \left\| \overline{\mathsf{H}}_k^{-1} \right\|^2 \|\bm{m}_k \|^2 \right],
\end{align*}
where 
\begin{align*}
\left\| \overline{\mathsf{H}}_k^{-1} \right\| = \left\| \mathsf{diag}\left(h_{k,i}^{-\frac{1}{2}} \right) \right\| = {\max_{i\in [d]} h_{k,i}^{-\frac{1}{2}}} 
\end{align*}
and ${\tilde{\gamma}}_k := 1 - {\gamma}^{k+1} \geq 1 - {\gamma}$. Moreover, (A1) ensures that, for all $k \in \mathbb{N}$, 
\begin{align*}
h_{k,i} \geq h_{0,i} \geq h_0^* := \min_{i\in [d]} h_{0,i}.
\end{align*}
Hence, (\ref{induction}) implies that, for all $k\in \mathbb{N}$,
\begin{align*}
\mathbb{E} \left[\| \bm{\mathsf{d}}_k \|_{\mathsf{H}_k}^2 \right] \leq 
\frac{1}{(1-{\gamma})^2 h_0^*} \left( \frac{\sigma^2}{b} + P^2 \right),
\end{align*}
completing the proof.
\end{proof}

We are now in the position to prove the following theorems.

\begin{thm}\label{thm:main}
Suppose that (S1)--(S3), (A1)--(A2), and (L1)--(L2) hold and consider Algorithm \ref{algo:1}. Then, for all $\bm{\theta} \in \mathcal{S}$ and all $K\geq 1$,
\begin{align*}
\mathrm{V}(K,\bm{\theta})
\leq
\frac{d \mathrm{Dist}(\bm{\theta}) H}{2 \alpha \tilde{\beta} K}
+
\frac{\alpha}{2 \tilde{\beta} \tilde{\gamma}^2 h_0^*} \left( \frac{\sigma^2}{b} + P^2 \right)
+
\frac{\beta \sqrt{d\mathrm{Dist}(\bm{\theta})}}{\tilde{\beta}} \sqrt{\frac{\sigma^2}{b} + P^2},
\end{align*}
where $\tilde{\gamma} := 1 - \gamma$, $\mathrm{Dist}(\bm{\theta})$ and $H_i$ are defined as in (L2) and (A2), and $H := \max_{i\in [d]} H_i$.
\end{thm}

\begin{proof}
Let $\bm{\theta}\in \mathcal{S}$. From the definition of $\Gamma_K$ in Lemma \ref{lem:key}, we have that
\begin{align}\label{LAM}
\begin{split}
\Gamma_K
&=
\tilde{\gamma}_1 \mathbb{E}\left[\left\| \bm{\theta}_{1} - \bm{\theta} \right\|_{\mathsf{H}_{1}}^2\right]
+
\underbrace{
\sum_{k=2}^K \left\{
\tilde{\gamma}_k \mathbb{E}\left[\left\| \bm{\theta}_{k} - \bm{\theta} \right\|_{\mathsf{H}_{k}}^2\right]
-
\tilde{\gamma}_{k-1} \mathbb{E}\left[\left\| \bm{\theta}_{k} - \bm{\theta} \right\|_{\mathsf{H}_{k-1}}^2\right] 
\right\}
}_{\tilde{\Gamma}_K}\\
&\quad-
\tilde{\gamma}_K \mathbb{E} \left[ \left\| \bm{\theta}_{K+1} - \bm{\theta} \right\|_{\mathsf{H}_{K}}^2 \right].
\end{split}
\end{align}
Since $\overline{\mathsf{H}}_k \in \mathbb{S}_{++}^d$ exists such that $\mathsf{H}_k = \overline{\mathsf{H}}_k^2$, we have $\|\bm{x}\|_{\mathsf{H}_k}^2 = \| \overline{\mathsf{H}}_k \bm{x} \|^2$ for all $\bm{x}\in\mathbb{R}^d$. Accordingly, we have 
\begin{align*}
\tilde{\Gamma}_K 
=
\mathbb{E} \left[ 
\sum_{k=2}^K 
\left\{
\tilde{\gamma}_k \left\| \overline{\mathsf{H}}_{k} (\bm{\theta}_{k} - \bm{\theta}) \right\|^2
-
\tilde{\gamma}_{k-1} \left\| \overline{\mathsf{H}}_{k-1} (\bm{\theta}_{k} - \bm{\theta}) \right\|^2
\right\}
\right].
\end{align*}
From $\overline{\mathsf{H}}_{k} = \mathsf{diag}(\sqrt{h_{k,i}})$, we have that, for all $\bm{x} = (x_i)_{i=1}^d \in \mathbb{R}^d$, $\| \overline{\mathsf{H}}_{k} \bm{x} \|^2 = \sum_{i=1}^d h_{k,i} x_i^2$. Hence, for all $K\geq 2$,
\begin{align}\label{DELTA}
\tilde{\Gamma}_K 
= 
\mathbb{E} \left[ 
\sum_{k=2}^K
\sum_{i=1}^d 
\left(
\tilde{\gamma}_k h_{k,i}
-
\tilde{\gamma}_{k-1} h_{k-1,i}
\right)
(\theta_{k,i} - \theta_i)^2
\right].
\end{align}
From (A1), we have that, for all $k \geq 1$ and all $i\in [d]$,
\begin{align*}
\tilde{\gamma}_k h_{k,i} - \tilde{\gamma}_{k-1} h_{k-1,i} \geq 0.
\end{align*} 
Moreover, from (L2), $\max_{i \in [d]} \sup_{k\in\mathbb{N}} (\theta_{k,i} - \theta_i)^2 \leq \mathrm{Dist}(\bm{\theta})$. Accordingly, for all $K \geq 2$,
\begin{align*}
\tilde{\Gamma}_K
\leq
\mathrm{Dist}(\bm{\theta})
\mathbb{E} \left[ 
\sum_{k=2}^K
\sum_{i=1}^d 
\left(
\tilde{\gamma}_k h_{k,i}
-
\tilde{\gamma}_{k-1} h_{k-1,i}
\right)
\right]
= 
\mathrm{Dist}(\bm{\theta})
\mathbb{E} \left[ 
\sum_{i=1}^d
\left(
\tilde{\gamma}_K h_{K,i}
-
\tilde{\gamma}_1 h_{1,i}
\right)
\right].
\end{align*}
Therefore, (\ref{LAM}), $\tilde{\gamma}_1 \mathbb{E} [\| \bm{\theta}_{1} - \bm{\theta}\|_{\mathsf{H}_{1}}^2] \leq \mathrm{Dist}(\bm{\theta}) \tilde{\gamma}_1 \mathbb{E} [ \sum_{i=1}^d h_{1,i}]$, and (A2) imply, for all $K\geq 1$,
\begin{align*}
\Gamma_K
&\leq
\tilde{\gamma}_1 \mathrm{Dist}(\bm{\theta}) \mathbb{E} \left[ 
\sum_{i=1}^d h_{1,i} \right]
+
\mathrm{Dist}(\bm{\theta})
\mathbb{E} \left[
\sum_{i=1}^d 
\left(
\tilde{\gamma}_K h_{K,i}
-
\tilde{\gamma}_1 h_{1,i}
\right)
\right]\\
&=
\tilde{\gamma}_K \mathrm{Dist}(\bm{\theta})
\mathbb{E} \left[
\sum_{i=1}^d 
h_{K,i}
\right]\\
&\leq 
\tilde{\gamma}_K \mathrm{Dist}(\bm{\theta}) 
\sum_{i=1}^d 
H_{i}\\
&\leq 
\tilde{\gamma}_K \mathrm{Dist}(\bm{\theta}) 
d H,
\end{align*}
where $H = \max_{i\in [d]} H_i$. From $\tilde{\gamma}_K = 1 - {\gamma}^{K+1} \leq 1$, we have that 
\begin{align}\label{L} 
\frac{1}{2 \alpha \tilde{\beta}} \Gamma_K 
\leq 
\frac{d \mathrm{Dist}(\bm{\theta}) H}{2 \alpha \tilde{\beta}}.
\end{align}
Lemma \ref{lem:bdd} implies that, for all $K\geq 1$,
\begin{align*}
A_K := \sum_{k=1}^K \tilde{\gamma}_k \mathbb{E} \left[ \left\|\bm{\mathsf{d}}_k \right\|_{\mathsf{H}_k}^2 \right] 
\leq 
\sum_{k=1}^K
\frac{\tilde{\gamma}_k}{(1-{\gamma})^2 h_0^*} \left( \frac{\sigma^2}{b} + P^2 \right),
\end{align*}
which, together with $\tilde{\gamma}_k \leq 1$, implies that
\begin{align}\label{D}
\frac{\alpha}{2 \tilde{\beta}} A_K 
\leq
\frac{\alpha K}{2 \tilde{\beta} (1-{\gamma})^2 h_0^*} \left( \frac{\sigma^2}{b} + P^2 \right).
\end{align}
Lemma \ref{lem:bdd} and Jensen's inequality ensure that, for all $k\in\mathbb{N}$,
\begin{align*}
\mathbb{E}\left[ \|\bm{m}_k\| \right] 
\leq 
\sqrt{\frac{\sigma^2}{b} + P^2}.
\end{align*}
The Cauchy-Schwarz inequality and (L2) guarantee that, for all $K\geq 1$,
\begin{align}\label{B}
\begin{split}
\frac{\beta}{\tilde{\beta}} B_K 
&:= 
\frac{\beta}{\tilde{\beta}}
\sum_{k=1}^K \mathbb{E} \left[ (\bm{\theta} - \bm{\theta}_k)^\top \bm{m}_{k-1} \right]
\leq 
\frac{\beta}{\tilde{\beta}} \sqrt{d\mathrm{Dist}(\bm{\theta})} \sum_{k=1}^K \mathbb{E} \left[
\left\|\bm{m}_{k-1} \right\|
\right]\\
&\leq
\frac{\beta \sqrt{d\mathrm{Dist}(\bm{\theta})}}{\tilde{\beta}} K \sqrt{\frac{\sigma^2}{b} + P^2}.
\end{split}
\end{align}
Therefore, (\ref{key}), (\ref{L}), (\ref{D}), and (\ref{B}) lead to the assertion in Theorem \ref{thm:main}. This completes the proof.
\end{proof}

\begin{thm}\label{thm:main2}
Suppose that (S1)--(S3), (A1), (L1), and (U1)--(U3) hold and consider Algorithm \ref{algo:1}. Then, for all $K \geq 1$,
\begin{align*}
\mathrm{V}(K,\bm{\theta}^*) 
\geq 
\frac{X (\bm{\theta}^*)}{2 \alpha \tilde{\beta} K}
+ 
\frac{\sigma^2 (c_1 \alpha \tilde{\gamma} - 2 c_2 \beta)}{2 \tilde{\beta}b}
-
\frac{c_2 \beta P^2}{\tilde{\beta}},
\end{align*}
where $X(\bm{\theta}^*)$, $c_1$, and $c_2$ are defined by (U1)--(U3), and $\sigma^2$ is defined by (S2)\eqref{sigma}.
\end{thm}

\begin{proof}
Condition (A1), together with \eqref{LAM} and \eqref{DELTA}, implies that $\tilde{\Gamma}_K \geq 0$ for all $K \geq 1$. Hence, from \eqref{LAM} and $1 \geq \tilde{\gamma}_k \geq \tilde{\gamma}_{k-1}$ ($k \geq 1$), 
\begin{align*}
\Gamma_K
&\geq 
\tilde{\gamma}_1 \mathbb{E}\left[\left\| \bm{\theta}_{1} - \bm{\theta}^* \right\|_{\mathsf{H}_{1}}^2\right]
-
\tilde{\gamma}_K \mathbb{E} \left[ \left\| \bm{\theta}_{K+1} - \bm{\theta}^* \right\|_{\mathsf{H}_{K}}^2 \right]\\
&\geq 
(1 - \gamma) \mathbb{E}\left[\left\| \bm{\theta}_{1} - \bm{\theta}^* \right\|_{\mathsf{H}_{1}}^2\right]
-
\mathbb{E} \left[ \left\| \bm{\theta}_{K+1} - \bm{\theta}^* \right\|_{\mathsf{H}_{K}}^2 \right],
\end{align*}
which, together with (U3), implies that, for all $K \geq 1$, 
\begin{align*}
\Gamma_K
&\geq
X (\bm{\theta}^*).
\end{align*}
Moreover, (U1) and (U2) imply that, for all $K \geq 1$,
\begin{align*}
&A_K = \sum_{k=1}^K \tilde{\gamma}_k \mathbb{E}\left[ \|\bm{\mathsf{d}}_k\|_{\mathsf{H}_k}^2 \right]
\geq \frac{c_1 \tilde{\gamma} \sigma^2}{b} K
\text{ and}\\ 
&B_K = \sum_{k=1}^K \mathbb{E}\left[ (\bm{\theta}^* - \bm{\theta}_k)^\top \bm{m}_{k-1} \right] \geq - c_2 \left( \frac{\sigma^2}{b} + P^2  \right) K.
\end{align*}
Accordingly, Lemma \ref{lem:key} leads to the assertion in Theorem \ref{thm:main2}. This completes the proof.
\end{proof}

\subsection{Proof of Theorem \ref{thm:1}}
\begin{proof}
Theorem \ref{thm:main} guarantees that, for all $\bm{\theta} \in \mathcal{S}$ and all $K \geq 1$,
\begin{align*}
\mathrm{V}(K,\bm{\theta})
\leq
\underbrace{\frac{d \mathrm{Dist}(\bm{\theta}) H}{2 \alpha \tilde{\beta}}}_
{A(\alpha,\beta,\bm{\theta},d,H)}
\frac{1}{K}
+
\underbrace{\frac{\sigma^2 \alpha}{2 \tilde{\beta} \tilde{\gamma}^2 h_0^*}}_
{B(\alpha,\beta,\gamma,\sigma^2,h_0^*)} \frac{1}{b}
+
\underbrace{\frac{P^2 \alpha}{2 \tilde{\beta} \tilde{\gamma}^2 h_0^*}}_{C(\alpha,\beta,\gamma,P^2,h_0^*)} 
+
\underbrace{\frac{\sqrt{d\mathrm{Dist}(\bm{\theta})(\sigma^2 + P^2)} \beta}{\tilde{\beta}}}_{D(\beta,\bm{\theta},d,\sigma^2, P^2)}.
\end{align*}
If the right-hand side of the above inequality is less than or equal to $\epsilon^2$, then we have that 
\begin{align*}
\frac{A}{K} + \frac{B}{b} + C + D \leq \epsilon^2,
\end{align*}
which implies that 
\begin{align*}
\underline{K}(b) := \frac{A b}
{\{\epsilon^2 - (C+D)\} b - B} \leq K,
\end{align*}
where $b > B/\{(\epsilon^2 - (C+D)\} \geq 0$. We have that, for $b > B/\{(\epsilon^2 - (C+D)\}$,
\begin{align*}
&\underline{K}^\star := 
\frac{A}{\epsilon^2 - (C+D)} < \underline{K}(b),\\
&\frac{\mathrm{d} \underline{K}(b)}{\mathrm{d} b}
= 
\frac{- A B}{[\{\epsilon^2 - (C+D)\} b - B]^2} \leq 0,\\
&\frac{\mathrm{d}^2 \underline{K}(b)}{\mathrm{d} b^2}
= 
\frac{2 A B \{\epsilon^2 - (C+D)\}}{[\{\epsilon^2 - (C+D)\} b - B]^3} \geq 0.
\end{align*}
Hence, $\underline{K}$ is monotone decreasing and convex for $b > B/\{(\epsilon^2 - (C+D)\}$. This completes the proof.
\end{proof}

\subsection{Proof of Theorem \ref{thm:1_1}}
\begin{proof}
We have that, for $b > B/\{(\epsilon^2 - (C+D)\}$, 
\begin{align*}
\underline{K}(b) b = 
\frac{A b^2}
{\{\epsilon^2 - (C+D)\} b - B}.
\end{align*}
Hence, 
\begin{align*}
&\frac{\mathrm{d} \underline{K}(b) b}{\mathrm{d} b}
= 
\frac{Ab[\{(\epsilon^2 - (C+D)\}b - 2 B]}
{[\{(\epsilon^2 - (C+D)\}b - B]^2},\\
&\frac{\mathrm{d}^2 \underline{K}(b) b}{\mathrm{d} b^2}
= 
\frac{2 AB^2}{[\{(\epsilon^2 - (C+D)\}b - B]^3} \geq 0,
\end{align*}
which implies that $\underline{K}(b) b$ is convex for $b > B/\{(\epsilon^2 - (C+D)\}$ and 
\begin{align*}
\frac{\mathrm{d} \underline{K}(b) b}{\mathrm{d} b}
\begin{cases}
< 0 &\text{ if } b < b_*,\\
= 0 &\text{ if } b = b_* = \frac{2B}{\epsilon^2 - (C+D)},\\
> 0 &\text{ if } b > b_*.
\end{cases}
\end{align*}
The point $b_*$ attains the minimum value $\underline{K}(b_*) b_*$ of $\underline{K}(b) b$. This completes the proof.
\end{proof}

\subsection{Proof of Theorem \ref{thm:2}}
\begin{proof}
Theorem \ref{thm:main2} guarantees that, for all $K \geq 1$,
\begin{align*}
\mathrm{V}(K,\bm{\theta}^*) 
\geq 
\underbrace{\frac{X (\bm{\theta}^*)}{2 \alpha \tilde{\beta}}}_{E(\alpha,\beta,\bm{\theta}^*)} \frac{1}{K}
+ 
\underbrace{\frac{\sigma^2 (c_1 \alpha \tilde{\gamma} - 2 c_2 \beta)}{2 \tilde{\beta}}}_{F(\alpha,\beta,\gamma,\sigma^2,c_1,c_2)} \frac{1}{b}
-
\underbrace{\frac{c_2 \beta P^2}{\tilde{\beta}}}_{G(\beta,c_2,P^2)}.
\end{align*}
If the right-hand side of the above inequality is more than or equal to $\delta \epsilon^2$, then we have that
\begin{align*}
\frac{E}{K} + \frac{F}{b} - G \geq \delta \epsilon^2,
\end{align*}
which implies that 
\begin{align*}
\overline{K}(b) := \frac{Eb}{(\delta \epsilon^2 + G) b - F} \geq K,
\end{align*}
where $b > F/(\delta \epsilon^2 + G) \geq 0$. We have that, for $b > F/(\delta \epsilon^2 + G)$, 
\begin{align*}
&\overline{K}^\star := \frac{E}{\delta \epsilon^2 + G} < \overline{K}(b),\\
&\frac{\mathrm{d}\overline{K}(b)}{\mathrm{d} b}
= \frac{-EF}{\{(\delta \epsilon^2 + G) b - F\}^2} \leq 0,\\
&\frac{\mathrm{d}^2 \overline{K}(b)}{\mathrm{d} b^2}
= \frac{2EF (\delta \epsilon^2 + G)}{\{(\delta \epsilon^2 + G) b - F\}^3} \geq 0.
\end{align*}
Hence, $\overline{K}$ is monotone decreasing and convex for $b > F/(\delta \epsilon^2 + G)$. This completes the proof. 
\end{proof}

\subsection{Proof of Theorem \ref{thm:2_1}}
\begin{proof}
We have that, for $b > F/(\delta \epsilon^2 + G)$,
\begin{align*}
\overline{K}(b)b = \frac{Eb^2}{(\delta \epsilon^2 + G) b - F}.
\end{align*}
Hence, 
\begin{align*}
&\frac{\mathrm{d} \overline{K}(b)b}{\mathrm{d}b}
= 
\frac{Eb\{ (\delta \epsilon^2 + G) b - 2F\}}{\{(\delta \epsilon^2 + G) b - F\}^2},\\
&\frac{\mathrm{d}^2 \overline{K}(b)b}{\mathrm{d}b^2}
= 
\frac{2EF^2}{\{(\delta \epsilon^2 + G) b - F\}^3} \geq 0,
\end{align*}
which implies that $\overline{K}(b)b$ is convex for $b > F/(\delta \epsilon^2 + G)$ and 
\begin{align*}
\frac{\mathrm{d}\overline{K}(b)b}{\mathrm{d}b}
\begin{cases}
< 0 &\text{ if } b < b^*,\\
= 0 &\text{ if } b = b^* = \frac{2F}{\delta \epsilon^2 + G},\\
> 0 &\text{ if } b > b^*.
\end{cases}
\end{align*}
The point $b^*$ attains the minimum value $\overline{K}(b^*)b^*$ of $\overline{K}(b)b$. This completes the proof.
\end{proof}


\begin{thebibliography}{10}

\bibitem{ar2017}
Martin Arjovsky, Soumith Chintala, and L\'eon Bottou.
\newblock Wasserstein {GAN}.
\newblock \url{https://arxiv.org/pdf/1701.07875.pdf}, 2017.

\bibitem{bottou}
L\'eon Bottou, Frank~E. Curtis, and Jorge Nocedal.
\newblock Optimization methods for large-scale machine learning.
\newblock {\em SIAM Review}, 60:223--311, 2018.

\bibitem{chen2020}
Hao Chen, Lili Zheng, Raed AL~Kontar, and Garvesh Raskutti.
\newblock Stochastic gradient descent in correlated settings: {A} study on
  {G}aussian processes.
\newblock In {\em Advances in Neural Information Processing Systems},
  volume~33, 2020.

\bibitem{chen2019}
Xiangyi Chen, Sijia Liu, Ruoyu Sun, and Mingyi Hong.
\newblock On the convergence of a class of {A}dam-type algorithms for
  non-convex optimization.
\newblock In {\em Proceedings of The International Conference on Learning
  Representations}, 2019.

\bibitem{adagrad}
John Duchi, Elad Hazan, and Yoram Singer.
\newblock Adaptive subgradient methods for online learning and stochastic
  optimization.
\newblock {\em Journal of Machine Learning Research}, 12:2121--2159, 2011.

\bibitem{facc1}
F.~Facchinei and J.-S. Pang.
\newblock {\em Finite-Dimensional Variational Inequalities and Complementarity
  Problems I}.
\newblock Springer, New York, 2003.

\bibitem{spider}
Cong Fang, Chris~Junchi Li, Zhouchen Lin, and Tong Zhang.
\newblock {SPIDER}: {N}ear-optimal non-convex optimization via stochastic
  path-integrated differential estimator.
\newblock In {\em Advances in Neural Information Processing Systems},
  volume~31, 2018.

\bibitem{feh2020}
Benjamin Fehrman, Benjamin Gess, and Arnulf Jentzen.
\newblock Convergence rates for the stochastic gradient descent method for
  non-convex objective functions.
\newblock {\em Journal of Machine Learning Research}, 21:1--48, 2020.

\bibitem{horn}
Roger~A. Horn and Charles~R. Johnson.
\newblock {\em Matrix Analysis}.
\newblock Cambridge University Press, Cambridge, 1985.

\bibitem{iiduka2021}
Hideaki Iiduka.
\newblock Appropriate learning rates of adaptive learning rate optimization
  algorithms for training deep neural networks.
\newblock {\em IEEE Transactions on Cybernetics}, DOI:
  10.1109/TCYB.2021.3107415, 2021.

\bibitem{adam}
Diederik~P. Kingma and Jimmy Ba.
\newblock Adam: A method for stochastic optimization.
\newblock In {\em Proceedings of The International Conference on Learning
  Representations}, 2015.

\bibitem{loizou2021}
Nicolas Loizou, Sharan Vaswani, Issam Laradji, and Simon Lacoste-Julien.
\newblock Stochastic polyak step-size for {SGD}: {A}n adaptive learning rate
  for fast convergence.
\newblock In {\em Proceedings of the 24th International Conference on
  Artificial Intelligence and Statistics}, volume 130, 2021.

\bibitem{luo2019}
Liangchen Luo, Yuanhao Xiong, Yan Liu, and Xu~Sun.
\newblock Adaptive gradient methods with dynamic bound of learning rate.
\newblock In {\em Proceedings of The International Conference on Learning
  Representations}, 2019.

\bibitem{kfac}
James Martens and Roger Grosse.
\newblock Optimizing neural networks with {K}ronecker-factored approximate
  curvature.
\newblock In {\em Proceedings of Machine Learning Research}, volume~37, pages
  2408--2417, 2015.

\bibitem{dun2020}
Celestine Mendler-D\"{u}nner, Juan~C. Perdomo, Tijana Zrnic, and Moritz Hardt.
\newblock Stochastic optimization for performative prediction.
\newblock In {\em Advances in Neural Information Processing Systems},
  volume~33, 2020.

\bibitem{nem2009}
Arkadi Nemirovski, Anatoli Juditsky, Guanghui Lan, and Alexander Shapiro.
\newblock Robust stochastic approximation approach to stochastic programming.
\newblock {\em SIAM Journal on Optimization}, 19:1574--1609, 2009.

\bibitem{nest1983}
Yurii Nesterov.
\newblock A method for unconstrained convex minimization problem with the rate
  of convergence ${O}(1/k^2)$.
\newblock {\em Doklady AN USSR}, 269:543--547, 1983.

\bibitem{polyak1964}
Boris~T. Polyak.
\newblock Some methods of speeding up the convergence of iteration methods.
\newblock {\em USSR Computational Mathematics and Mathematical Physics},
  4:1--17, 1964.

\bibitem{reddi2018}
Sashank~J. Reddi, Satyen Kale, and Sanjiv Kumar.
\newblock On the convergence of {A}dam and beyond.
\newblock In {\em Proceedings of The International Conference on Learning
  Representations}, 2018.

\bibitem{momentum}
David~E. Rumelhart, Geoffrey~E. Hinton, and Ronald~J. Williams.
\newblock Learning representations by back-propagating errors.
\newblock {\em Nature}, 323:533--536, 1986.

\bibitem{sca2020}
Kevin Scaman and C\'edric Malherbe.
\newblock Robustness analysis of non-convex stochastic gradient descent using
  biased expectations.
\newblock In {\em Advances in Neural Information Processing Systems},
  volume~33, 2020.

\bibitem{Schmidt2021}
Robin~M. Schmidt, Frank Schneider, and Philipp Hennig.
\newblock Descending through a crowded valley--{B}enchmarking deep learning
  optimizers.
\newblock In {\em Proceedings of the 38th International Conference on Machine
  Learning}, volume 139, pages 9367--9376, 2021.

\bibitem{shallue2019}
Christopher~J. Shallue, Jaehoon Lee, Joseph Antognini, Jascha Sohl-Dickstein,
  Roy Frostig, and George~E. Dahl.
\newblock Measuring the effects of data parallelism on neural network training.
\newblock {\em Journal of Machine Learning Research}, 20:1--49, 2019.

\bibitem{l.2018dont}
Samuel~L. Smith, Pieter-Jan Kindermans, and Quoc~V. Le.
\newblock Don't decay the learning rate, increase the batch size.
\newblock In {\em Proceedings of The International Conference on Learning
  Representations}, 2018.

\bibitem{sut2013}
Ilya Sutskever, James Martens, George Dahl, and Geoffrey Hinton.
\newblock On the importance of initialization and momentum in deep learning.
\newblock In {\em Proceedings of the 30th International Conference on Machine
  Learning}, pages 1139--1147, 2013.

\bibitem{rmsprop}
Tijmen Tieleman and Geoffrey Hinton.
\newblock {RMSP}rop: {D}ivide the gradient by a running average of its recent
  magnitude.
\newblock {\em {COURSERA}: {N}eural networks for machine learning}, 4:26--31,
  2012.

\bibitem{vas2017}
Ashish Vaswani, Noam Shazeer, Niki Parmar, Jakob Uszkoreit, Llion Jones,
  Aidan~N. Gomez, Lukasz Kaiser, and Illia Polosukhin.
\newblock Attention is {A}ll you {N}eed.
\newblock In {\em Advances in Neural Information Processing Systems},
  volume~30, 2017.

\bibitem{kxu2015}
Kelvin Xu, Jimmy Ba, Ryan Kiros, Kyunghyun Cho, Aaron Courville, Ruslan
  Salakhudinov, Rich Zemel, and Yoshua Bengio.
\newblock Show, attend and tell: {N}eural image caption generation with visual
  attention.
\newblock In {\em Proceedings of the 32nd International Conference on Machine
  Learning}, volume~37, pages 2048--2057, 2015.

\bibitem{zhang2019}
Guodong Zhang, Lala Li, Zachary Nado, James Martens, Sushant Sachdeva,
  George~E. Dahl, Christopher~J. Shallue, and Roger Grosse.
\newblock Which algorithmic choices matter at which batch sizes? {I}nsights
  from a noisy quadratic model.
\newblock In {\em Advances in Neural Information Processing Systems},
  volume~32, 2019.

\bibitem{ada}
Juntang Zhuang, Tommy Tang, Yifan Ding, Sekhar Tatikonda, Nicha Dvornek,
  Xenophon Papademetris, and James~S. Duncan.
\newblock Ada{B}elief optimizer: {A}dapting stepsizes by the belief in observed
  gradients.
\newblock In {\em Advances in Neural Information Processing Systems},
  volume~33, 2020.

\bibitem{zin2010}
Martin Zinkevich, Markus Weimer, Lihong Li, and Alex Smola.
\newblock Parallelized stochastic gradient descent.
\newblock In {\em Advances in Neural Information Processing Systems},
  volume~23, 2010.

\end{thebibliography}
\end{document}